\documentclass[runningheads]{llncs} 
\usepackage{times}
\usepackage{helvet}       
\usepackage{comment}
\usepackage{enumitem}
\usepackage{float}
\usepackage{booktabs}
\usepackage{cuted}
\usepackage{algorithm, algpseudocode}
\usepackage{bm}             
\usepackage{courier}
\usepackage{subcaption}
\usepackage[hyphens]{url}     
\usepackage{graphicx}
\usepackage{pifont}
\usepackage{amssymb,amsmath}

\usepackage{tikz}

\usetikzlibrary{automata,positioning}
\usepackage{mathtools}
\usepackage{tabularx}
\usepackage{tikz}
\usepackage{subfloat}
\usepackage[skins]{tcolorbox}

\usetikzlibrary{arrows, automata, positioning}
\tikzset{auto, >=stealth}
\tikzset{every edge/.append style={shorten >= 1pt}}
\tikzset{main node/.style={circle,draw,minimum size=1cm,inner sep=0pt},
}

\usepackage{etoolbox}

\providetoggle{long}
\settoggle{long}{false}

\newtcbox{\remind}{%
  enhanced jigsaw,nobeforeafter,size=fbox,sharp corners,
  shrink tight,
  extrude by=3pt,
  tcbox raise base,
  borderline={0.5pt}{-1pt}{red,opacity=0.75},
  opacityframe=0.75,
  opacityback=0.5,
}

\newcommand{\argmin}{\mathrm{arg}\min}

\newcommand{\set}[1]{\{ #1 \}}

\newcommand{\methodA}{JIRP}

\newcommand{\Office}{Office World Scenario}
\newcommand{\office}{office world scenario}
\newcommand{\OfficeA}{Task 1}

\newcommand{\OfficeB}{Task 2}

\newcommand{\OfficeC}{Task 3}

\newcommand{\craft}{Minecraft world scenario}

\newcommand{\qValue}{q}

\newcommand{\implement}{encode}

\newcommand{\algoName}{JIRP} 
\newcommand{\newAlgoName}{AFRAI-RL}

\newcommand{\init}{I}

\newcommand{\dfaStates}{V}
\newcommand{\dfaState}{v}

\newcommand{\dfaTransition}{\delta}
\newcommand{\dfaInputAlphabet}{\Sigma}

\newcommand{\dfaFinalStates}{F}
\newcommand{\dfa}{\mathfrak A}

\newcommand{\mdp}{\mathcal{M}} 
\newcommand{\mdpStates}{X}
\newcommand{\mdpState}{x}
\newcommand{\mdpInit}{\mdpState_{init}}
\newcommand{\mdpActions}{A}
\newcommand{\mdpAction}{a}
\newcommand{\mdpProb}{p}
\newcommand{\mdpRewardFunction}{R}

\newcommand{\mdpLabel}{\ell}
\newcommand{\mdpRewards}{r}
\newcommand{\trajectory}[1]{\ensuremath{\mdpState_0 \mdpAction_0\ldots \mdpState_{#1} \mdpAction_{#1} \mdpState_{#1 + 1}}}

\newcommand{\rmLabels}{\mathcal P}
\newcommand{\rmLabelingFunction}{L}

\newcommand{\rmInputAlphabet}{2^\rmLabels}
\newcommand{\rmOutputAlphabet}{\mathcal{R}}

\newcommand{\machine}{\mathcal A}
\newcommand{\mealyStates}{W}

\newcommand{\mealyState}{w}
        
\newcommand{\mealyInputAlphabet}{\ensuremath{{2^\mathcal{P}}}}
\newcommand{\mealyInit}{{\mealyState_{init}}}

\newcommand{\mealyOutputAlphabet}{\ensuremath{\mathcal{R}}}
\newcommand{\mealyOutput}{\eta}
\newcommand{\mealyTransition}{\delta}
\newcommand{\inputTrace}{\lambda}
\newcommand{\outputTrace}{\rho}

\begin{document} 
\title{Active Finite Reward Automaton Inference and Reinforcement Learning Using Queries and Counterexamples}
 \author{
	Zhe Xu\inst{1}, Bo Wu\inst{2}, Aditya Ojha\inst{2}, Daniel Neider\inst{3}, Ufuk Topcu\inst{2} \\
}
\institute{
Arizona State University, Tempe AZ, USA
\email{xzhe1@asu.edu} \and
University of Texas at Austin, Austin TX, USA \email{bowu86@gmail.com, ~~~~~~~~~\{adiojha629, utopcu\}@utexas.edu}\and
Max Planck Institute for Software Systems, Kaiserslautern, Germany \email{neider@mpi-sws.org} 
}

\maketitle
\begin{abstract} 
Despite the fact that deep reinforcement learning (RL) has surpassed human-level performances in various tasks, it still has several fundamental challenges. First, most RL methods require intensive data from the exploration of the environment to achieve satisfactory performance. Second, the use of neural networks in RL renders it hard to interpret the internals of the system in a way that humans can understand. To address these two challenges, we propose a framework that enables an RL agent to reason over its exploration process and distill high-level knowledge for effectively guiding its future explorations. Specifically, we propose a novel RL algorithm that learns high-level knowledge in the form of a \textit{finite reward automaton} by using the L* learning algorithm. We prove that in episodic RL, a finite reward automaton can express any non-Markovian bounded reward functions with finitely many reward values and approximate any non-Markovian bounded reward function (with infinitely many reward values) with arbitrary precision. We also provide a lower bound for the episode length such that the proposed RL approach almost surely converges to an optimal policy in the limit. We test this approach on two RL environments with non-Markovian reward functions, choosing a variety of tasks with increasing complexity for each environment. We compare our algorithm with the state-of-the-art RL algorithms for non-Markovian reward functions, such as Joint Inference of Reward machines and Policies for RL (JIRP), Learning Reward Machine (LRM), and Proximal Policy Optimization (PPO2). Our results show that our algorithm converges to an optimal policy faster than other baseline methods. 
\end{abstract}


\section{Introduction}

Despite the fact that deep reinforcement learning (RL) has surpassed human-level
performances in various tasks, it still has several fundamental challenges. First, most RL methods require intensive data from the exploration of the environment to achieve satisfactory performance. Second, the use of neural networks in RL renders it hard to interpret the internals of the system in a way that humans can understand \cite{HOLZINGER202128}. 


To address these two challenges, we propose a framework that enables an RL agent to reason over its exploration process and distill high-level knowledge for effectively guiding its future explorations.	
Specifically, we learn high-level knowledge in the form of \textit{finite reward automata}, a type of Mealy machine that encodes non-Markovian
reward functions. The finite reward automata can be converted to a deterministic finite state machine, allowing a practitioner to more easily reason about what the agent is learning \cite{Hopcroft2006}. Thus, this representation is more interpretable than frameworks that use neural networks.

In comparison with other methods that also learn high-level knowledge during RL, the one proposed in this paper \textit{actively} infers a finite reward automaton from the RL episodes. We prove that in episodic RL, a finite reward automaton can express any non-Markovian bounded reward functions with finitely many reward values and approximate any non-Markovian bounded reward function (with infinitely many reward values) with arbitrary precision. As the learning agent infers this finite reward automaton during RL, it also performs RL (specifically, q-learning) to maximize its obtained rewards based on the inferred finite reward automaton. The inference method is inspired by the L* learning algorithm \cite{angluin1987learning}, and modified to the framework of RL. We maintain two q-functions, one for incentivizing the learning agent to answer the \textit{membership queries} during the explorations and the other one for obtaining optimal policies for the inferred finite reward automaton (in order to answer the \textit{equivalence queries}). Furthermore, we prove that the proposed RL approach almost surely converges to an optimal policy in the limit, if the episode length is longer than a theoretical lower bound value.




We implement the proposed approach and three baseline
methods (JIRP-SAT \cite{Xu2019jirp}, LRM-QRM \cite{Icarte2018}, and PPO2 \cite{schulman2017proximal}) in the Office world \cite{Xu2019jirp} and Minecraft world \cite{andreas2017modular} scenarios. The results show that, at worst the approach converges to an optimal policy 88.8\% faster than any of the baselines, and at best the approach converges while the other baselines do not.

\subsection{Related works} 
	
Our work is closely related to the use of formal methods in RL, such as RL for finite reward automata \cite{DBLP:conf/icml/IcarteKVM18} and RL with temporal logic specifications \cite{Aksaray2016,Li2017,Icarte2018,Fu2014ProbablyAC,Min2017,Alshiekh2018SafeRL}. The current methods assume that high-level knowledge, in the form of reward machines or temporal logic specifications, is known \textit{a priori}. However, in real-life use cases, such knowledge is implicit and must be inferred from data.


Towards the end of inferring high-level knowledge, several approaches have been proposed \cite{zhe_ijcai2019}, \cite{Xu2019jirp}, \cite{IcarteNIPS2019}, \cite{Neider2020AdviceGuidedRL}. In these methods, the agent jointly learns the high-level knowledge \textit{and} RL-policies \textit{concurrently}. In \cite{zhe_ijcai2019}, the inferred high-level knowledge is represented by \textit{temporal logic} formulas and used for RL-based transfer learning. The methods for inferring temporal logic formulas from data can be found in \cite{Hoxha2017,ATVA2021,Kong2017TAC,Bombara2016,Nasim2021,Neider,zhe_advisory,zhe2016,VazquezChanlatte2018LearningTS,zhe_info,NIPS2018Shah,zheADHS,zhe_info,zheletter2,cai2021modular,Furelos-Blanco_Law_Russo_Broda_Jonsson_2020}.


In comparison with temporal logic formulas, the finite reward automata used in this paper are more expressive in representing the high-level structural relationships. Moreover, even if the inferred finite reward automaton is incorrect during the first training loop, the agent is still able to self-correct and learn more complex tasks. In order to learn finite reward automata, the authors in \cite{Xu2019jirp} proposed using \textit{passive} inference of finite reward automata and utilizing the inferred finite reward automata to expedite RL. In \cite{IcarteNIPS2019}, the authors proposed a method to infer reward machines to represent the memories of Partially observable Markov decision processes (POMDP) and perform RL for the POMDP with the inferred reward machines.

In contrast, our method \textit{actively} infers the finite reward automaton in environments with non-Markovian reward functions. The active inference is facilitated by L* learning. This algorithm assumes the existence of a teacher who can answer the membership and equivalence queries \cite{angluin1987learning,wu2015counterexample,wu2018permissive,zhang2017supervisor}. In our approach, an RL engine fulfills the role of the teacher, and the queries are answered through interaction with the environment through the RL engine.

During the submission of this paper, an interesting method was proposed by the authors of \cite{Gaon_Brafman_2020} which also used L* learning for non-Markovian Rewards. While superficially similar to our work, the two approaches differ in three ways:

(1) The proposed approach in this paper uses finite reward automata, while \cite{Gaon_Brafman_2020} works with deterministic finite automata (DFAs) and general automata. This is a notable distinction because we prove that finite reward automata can express any non-Markovian bounded reward function with finitely many reward values in episodic RL.

(2) The authors of \cite{Gaon_Brafman_2020} use only maintain one type of DFA for answering the equivalence queries. The proposed approach in this paper maintains two types of finite reward automata, using one to answer equivalence queries, and the other to answer membership queries. This additional finite reward automaton can incentivize the agent to answer membership queries during the exploration.

(3) We provide a lower bound for the episode length such that the proposed RL approach almost surely converges to an optimal policy in the limit.

\section{Finite Reward Automata}
\label{sec:preliminaries}
In this section we introduce necessary background on reinforcement learning and finite reward automata.

\subsection{Markov Decision Processes and Finite Reward Automata}
\label{sec_MDP}
Let $\mathcal{M}= (\mdpStates, \mdpInit, \mdpActions, \mdpProb, \rmLabels, \mdpRewardFunction, \rmLabelingFunction)$ be a labeled \textit{Markov decision process} (labeled MDP) with a non-Markovian reward function, where the state space $\mdpStates$ and action set $\mdpActions$ are finite, $\mdpInit\in\mdpStates$ is a set of initial states, $\mdpProb: \mdpStates\times\mdpActions\times\mdpStates\rightarrow[0,1]$ is a probabilistic transition relation, $\rmLabels$ is a set of propositional variables (i.e., labels), $\mdpRewardFunction:(2^\rmLabels)^+\rightarrow\rmOutputAlphabet$ is a non-Markovian reward function where the reward depends on the history of the propositional variables (i.e., labels) that has been encountered, and $\rmLabelingFunction: \mdpStates\times \mdpActions\times \mdpStates\rightarrow 2^\rmLabels$ is a labeling function.   

We define the size of $\mdp$, denoted as $|\mdp|$, to be $|\mdpStates|$ (i.e., the cardinality of the set $\mdpStates$). A policy $\pi: \mdpStates\times\mdpActions\rightarrow [0,1]$ specifies the probability of taking each action for each state. The \textit{action-value function}, denoted as $q_{\pi}(\mdpState,\mdpAction)$, is the expected discounted reward if an agent applies policy $\pi$ after taking action $\mdpAction$ from state $\mdpState$. A finite sequence $\trajectory{k}$ generated by $\mathcal{M}$ under certain policy $\pi$ is called a \textit{trajectory}, starting from $\mdpState_1=\mdpInit$ and satisfies $\sum_{\mdpAction\in\mdpActions}\pi(\mdpState_k,\mdpAction)P(\mdpState_k,\mdpAction,\mdpState_{k+1})>0$ for all $k\ge1$. Its corresponding \emph{label sequence} is $\mdpLabel_0 \mdpLabel_1\ldots \mdpLabel_{k}$
where $\rmLabelingFunction(\mdpState_i, \mdpAction_{i}, \mdpState_{i+1}) = \mdpLabel_i$ for each $i \leq k$. Similarly, the corresponding \emph{reward sequence} is $\mdpRewards_1\ldots\mdpRewards_k$,
where $\mdpRewards_i = \mdpRewardFunction(\mdpLabel_0 \mdpLabel_1\ldots \mdpLabel_{k})$, for each $i \leq k$. 
We call the pair $(\inputTrace, \outputTrace):=(\mdpLabel_1\ldots\mdpLabel_k,\mdpRewards_1\ldots\mdpRewards_k)$ a \emph{trace}.

\begin{definition}
    \label{AllAttainableSequences}
    Let $\mathcal{M}= (\mdpStates, \mdpInit, \mdpActions, \mdpProb, \rmLabels, \mdpRewardFunction, \rmLabelingFunction)$ be a labeled \textit{Markov decision process}.
    We define a sequence  $(\mdpLabel_1,\mdpRewards_1),\ldots,(\mdpLabel_k,\mdpRewards_k)$ to be attainable if $k\leq eplength$ and $\mdpProb(\mdpState_i, \mdpAction_{i}, \mdpState_{i+1}) > 0$ for each $i \in \{0, \dots ,k\}$.
\end{definition}

\begin{definition}
\label{def:rewardMealyMachines}
A \emph{finite reward automaton} 
$\machine = (\mealyStates, \mealyInit, \mealyInputAlphabet, \mealyOutputAlphabet, \mealyTransition, \mealyOutput)$ consists of 
a finite, nonempty set $\mealyStates$ of states, 
an initial state $\mealyInit \in \mealyStates$, 
an input alphabet $\mealyInputAlphabet$,
an output alphabet $\mealyOutputAlphabet$, 
a (deterministic) transition function $\mealyTransition \colon \mealyStates \times \mealyInputAlphabet \to \mealyStates$, 
and an output function $\mealyOutput \colon \mealyStates \times \mealyInputAlphabet \to \mealyOutputAlphabet$, where $\mealyOutputAlphabet$ is a finite set of reward values ($\mealyOutputAlphabet\subset\mathbb{R}$).
We define the size of $\machine$, denoted as $|\machine|$, to be $|\mealyStates|$ (i.e., the cardinality of the set $\mealyStates$).
\end{definition}

\begin{remark}
A finite reward automaton is actually a Mealy machine (Shallit 2008) where the output alphabet is a finite set of values. When the output alphabet is an infinite set of values, it is called a reward machine in \cite{Xu2019jirp,DBLP:conf/icml/IcarteKVM18}.
\end{remark}

The run of a finite reward automaton $\machine$ on a sequence of labels $\mdpLabel_1\ldots \mdpLabel_k\in (\mealyInputAlphabet)^*$ is a sequence 
$\mealyState_0 (\mdpLabel_1, \mdpRewards_1) \mealyState_1 (\mdpLabel_2, \mdpRewards_2)\ldots \mealyState_{k-1}(\mdpLabel_k, \mdpRewards_k) \mealyState_{k}$ of states and label-reward pairs such that $\mealyState_0 = \mealyInit$
and for all $i\in\set{0,\ldots, k}$, we have $\mealyTransition(\mealyState_i, \mdpLabel_i) = \mealyState_{i+1}$ and $\mealyOutput(\mealyState_i,\mdpLabel_i) = \mdpRewards_i$.
We write $\machine[\mdpLabel_1\ldots\mdpLabel_k] = \mdpRewards_1\ldots\mdpRewards_k$ 
to connect the input label sequence 
to the sequence of rewards
produced by the machine $\machine$ \cite{Xu2019jirp}. 


\begin{definition}
We define that
a finite reward automaton $\machine$ \emph{\implement s} 
the reward function $\mdpRewardFunction$ of a labeled MDP $\mathcal{M}$
if for every trajectory $\trajectory{k}$ of finite length and the corresponding label sequence $\mdpLabel_1\ldots \mdpLabel_k$, 
the reward sequence equals $\machine[\mdpLabel_1\ldots \mdpLabel_k]$.
\end{definition}

\begin{definition}[Reward Product MDP]
Let $\mdp = (\mdpStates, \mdpInit, \mdpActions, \mdpProb, \rmLabels, \mdpRewardFunction, \rmLabelingFunction)$ be a labeled MDP with a non-Markovian reward function and $\machine = (\mealyStates, \mealyInit, \rmInputAlphabet, \rmOutputAlphabet, \mealyTransition, \mealyOutput)$ a finite reward automaton encoding its reward function.
We define the product MDP $\mdp_\machine = (\mdpStates', \mdpState'_{\init}, \mdpActions', \mdpProb', \rmLabels', \mdpRewardFunction', \rmLabelingFunction')$ by 
\begin{itemize}
	\item $\mdpStates' = \mdpStates \times \mealyStates$;
	\item $\mdpState'_{\init} = (\mdpInit, \mealyState_{init})$;
	\item $\mdpActions' = \mdpActions$; 
	\item $\mdpProb' \bigl( (\mdpState, \mealyState), \mdpAction, (\mdpState', \mealyState') \bigr)$\\
	$= \begin{cases}
	\mdpProb(\mdpState, \mdpAction, \mdpState') & \text{if}~\mealyState' = \mealyTransition(\mealyState, \rmLabelingFunction(\mdpState, \mdpAction, \mdpState'));\\
	0 & \text{otherwise};
	\end{cases}
	$
	\item $\rmLabels'=\rmLabels$;
	\item $\mdpRewardFunction' \bigl( (\mdpState, \mealyState), \mdpAction, (\mdpState', \mealyState') \bigr) = \mealyOutput \bigl(\mealyState, \rmLabelingFunction(\mdpState, \mdpAction, \mdpState') \bigr)$; and
	\item $\rmLabelingFunction'= \rmLabelingFunction$.
\end{itemize}
\label{product_MDP}
\end{definition} 
It can be seen from Definition \ref{product_MDP} that the reward function of the product MDP $\mdp_\machine$ is actually Markovian.

\subsection{Reinforcement Learning With Finite Reward Automata}

Q-learning \cite{Watkins1992} is a form of model-free reinforcement learning
(RL). 
Starting from state $\mdpState$, the system selects an action $\mdpAction$, which takes it to state $\mdpState'$ and obtains a reward $R$. The $\mathrm{Q}$-function values will be updated by the following rule:
\begin{align}
q(\mdpState, \mdpAction) \leftarrow (1-\alpha)q(\mdpState, \mdpAction)+\alpha(\mdpRewardFunction+\gamma\max_{\mdpAction}q(\mdpState', \mdpAction)).
\label{Q_update}
\end{align}                           

The q-learning algorithm can be modified to learn an optimal policy when the general reward function is encoded by a finite reward automaton
\cite{DBLP:conf/icml/IcarteKVM18}. Starting from state $(\mdpState, \mealyState)$ in the product space, the system selects an action $\mdpAction$, which takes it to state $(\mdpState',\mealyState')$ and obtains a reward $R$. The $\mathrm{Q}$-function values will be updated by the following rule. 
\begin{align}
q(\mdpState, \mealyState, \mdpAction) \leftarrow (1-\alpha)q(\mdpState, \mealyState, \mdpAction)+\alpha(\mdpRewardFunction+\gamma\max_{\mdpAction}q(\mdpState', \mealyState', \mdpAction)).
\label{Q_update2}
\end{align}  
We consider episodic Q-learning in this paper, and we use $eplength$ to denote the episode length.

\section{Expressivity of Finite Reward Automata}
In this section, we show that any non-Markovian reward function in episodic RL which has finitely many reward values can be encoded by finite reward automata, while any non-Markovian bounded reward function in episodic RL can be approximated by finite reward automata with arbitrary precision.

	\begin{theorem}
	For a labeled MDP $\mdp = (\mdpStates, \mdpInit, \mdpActions, \mdpProb, \rmLabels, \mdpRewardFunction, \rmLabelingFunction)$ with a non-Markovian reward function $\mdpRewardFunction:(2^\rmLabels)^+\rightarrow\rmOutputAlphabet_f$ in episodic RL, where $\rmOutputAlphabet_f$ is a finite set of values in $\mathbb{R}$, there exists at least one finite reward automaton $\machine = (\mealyStates, \mealyInit, \rmInputAlphabet, \rmOutputAlphabet, \mealyTransition, \mealyOutput)$ that can \implement\ the reward function $\mdpRewardFunction$.
	\end{theorem}
	\begin{proof} 
	We use $\mathcal{I}$ to denote the set of trajectories of length at most $\mathit{eplength}$ and $\mathcal{T}_i$ to denote the maximal time index for trajectory $i\in\mathcal{I}$. We construct a finite reward automaton $\machine = (\mealyStates, \mealyInit, \rmInputAlphabet, \rmOutputAlphabet, \mealyTransition, \mealyOutput)$, 
	\begin{itemize}
	\item $\mealyStates = \mealyInit\cup\{w_{i,t}\}_{i\in\mathcal{I}, 1<t\le \mathcal{T}_i}$, $w_{i,1}=\mealyInit$, $\forall i\in\mathcal{I}$;
	\item $\mealyOutputAlphabet=\mealyOutputAlphabet_f$;
	\item $\forall t\le\mathcal{T}_i, \mealyTransition\bigl( w_{i,t},\rmLabelingFunction( \mdpState_{i,t},\mdpAction_{i,t}, \mdpState_{i,t+1}) \bigr) = w_{i,t+1}$; and
	\item $\mealyOutput\bigl(w_{i,t}, \mdpLabel_i \bigr)=\mdpRewardFunction(\mdpLabel_{i,1},\dots,\mdpLabel_{i,t})$, where $\mdpLabel_{i,t}=\rmLabelingFunction( \mdpState_{i,t},\mdpAction_{i,t}, \mdpState_{i,t+1})$.
    \end{itemize}
    Then, it can be easily shown that for every trajectory $\trajectory{i,k}$ and corresponding label sequence $\mdpLabel_1\ldots \mdpLabel_k$ ($k\le\mathit{eplength}$), the reward sequence equals $\machine[\mdpLabel_1\ldots \mdpLabel_k]$, i.e., $\machine$ \implement s the reward function $\mdpRewardFunction$.
	\end{proof}
	\begin{remark}
	For a labeled MDP $\mdp = (\mdpStates, \mdpInit, \mdpActions, \mdpProb, \rmLabels, \mdpRewardFunction, \rmLabelingFunction)$ with a finite horizon, there can only be finitely many reward values, each corresponding to finitely many possible trajectories for the reward value. 
	\end{remark}

	\begin{theorem}
	For a labeled MDP $\mdp = (\mdpStates, \mdpInit, \mdpActions, \mdpProb, \rmLabels, \mdpRewardFunction, \rmLabelingFunction)$ with a non-Markovian bounded reward function $\mdpRewardFunction:(2^\rmLabels)^+\rightarrow\mathbb{R}$ in episodic RL, there exists at least one finite reward automaton $\machine = (\mealyStates, \mealyInit, \rmInputAlphabet, \rmOutputAlphabet, \mealyTransition, \mealyOutput)$ that can approximate the non-Markovian reward function $\mdpRewardFunction$ with arbitrary precision.
	\end{theorem}
	\begin{proof} 
	For a bounded reward function taking values in $[r_{\textrm{min}}, r_{\textrm{max}}]$, where $r_{\textrm{min}}, r_{\textrm{max}}\in\mathbb{R}$, $r_{\textrm{min}}\le r_{\textrm{max}}$, we construct a finite set $\mealyOutputAlphabet_{\epsilon}=\{r_{\textrm{min}}, r_{\textrm{min}}+\epsilon, r_{\textrm{min}}+2\epsilon, \dots, r_{\textrm{min}}+n_{\textrm{max}}\epsilon]$, where $\epsilon\in\mathbb{R}$, $\epsilon>0$, $n_{\textrm{max}}=\max\{n~\vert~r_{\textrm{min}}+n\epsilon\le r_{\textrm{max}}\}$. We use $\mathcal{I}$ to denote the set of trajectories of length at most $\mathit{eplength}$ generated from the labeled MDP $\mdp$ and $\mathcal{T}_i$ to denote the maximal time index for trajectory $i\in\mathcal{I}$. We construct a finite reward automaton $\machine = (\mealyStates, \mealyInit, \rmInputAlphabet, \rmOutputAlphabet, \mealyTransition, \mealyOutput)$, 
    where
	\begin{itemize}
	\item $\mealyStates = \mealyInit\cup\{w_{i,t}\}_{i\in\mathcal{I}, 1<t\le \mathcal{T}_i}$, $w_{i,1}=\mealyInit$, $\forall i\in\mathcal{I}$;
	\item $\mealyOutputAlphabet=\mealyOutputAlphabet_{\epsilon}$;
	\item $\forall t\le\mathcal{T}_i, \mealyTransition\bigl( w_{i,t}, \rmLabelingFunction( \mdpState_{i,t},\mdpAction_{i,t}, \mdpState_{i,t+1}) \bigr) = w_{i,t+1}$; and
	\item $\mealyOutput\bigl(w_{i,t}, \mdpLabel_{i,t} \bigr)=\argmin\limits_{r\in\mealyOutputAlphabet_{\epsilon}}\vert r-\mdpRewardFunction(\mdpLabel_{i,1},\dots,\mdpLabel_{i,t})\vert$, where $\mdpLabel_{i,t}=\rmLabelingFunction( \mdpState_{i,t},\mdpAction_{i,t}, \mdpState_{i,t+1})$.
    \end{itemize}
    Then, it can be easily shown that for any $\epsilon>0$ and every trajectory $\trajectory{i,k}$ ($k\le\mathit{eplength}$), we have $\vert\machine[\mdpLabel_{i,1},\dots,\mdpLabel_{i,t})]-\mdpRewardFunction[\mdpLabel_{i,1},\dots,\mdpLabel_{i,t}]\vert<\epsilon$, i.e., $\machine$ approximates the reward function $\mdpRewardFunction$ with arbitrary precision.
	\end{proof}

\section{Active Finite Reward Automaton Inference and Reinforcement Learning (\newAlgoName)}
In this section, we introduce the Active Finite Reward Automaton Inference and Reinforcement Learning (\newAlgoName) algorithm. Figure \ref{block} shows the block diagram of the \newAlgoName\ approach, and Algorithm \ref{whole_alg} shows the procedures of the \newAlgoName\ approach. The \newAlgoName\ approach consists of an active finite reward automaton inference engine and an RL engine. In the following two subsections, we will introduce the two engines and their interactions for obtaining the optimal RL policy for tasks with non-Markovian rewards.

	\begin{figure}[t]
		\centering
		\includegraphics[width=1\textwidth]{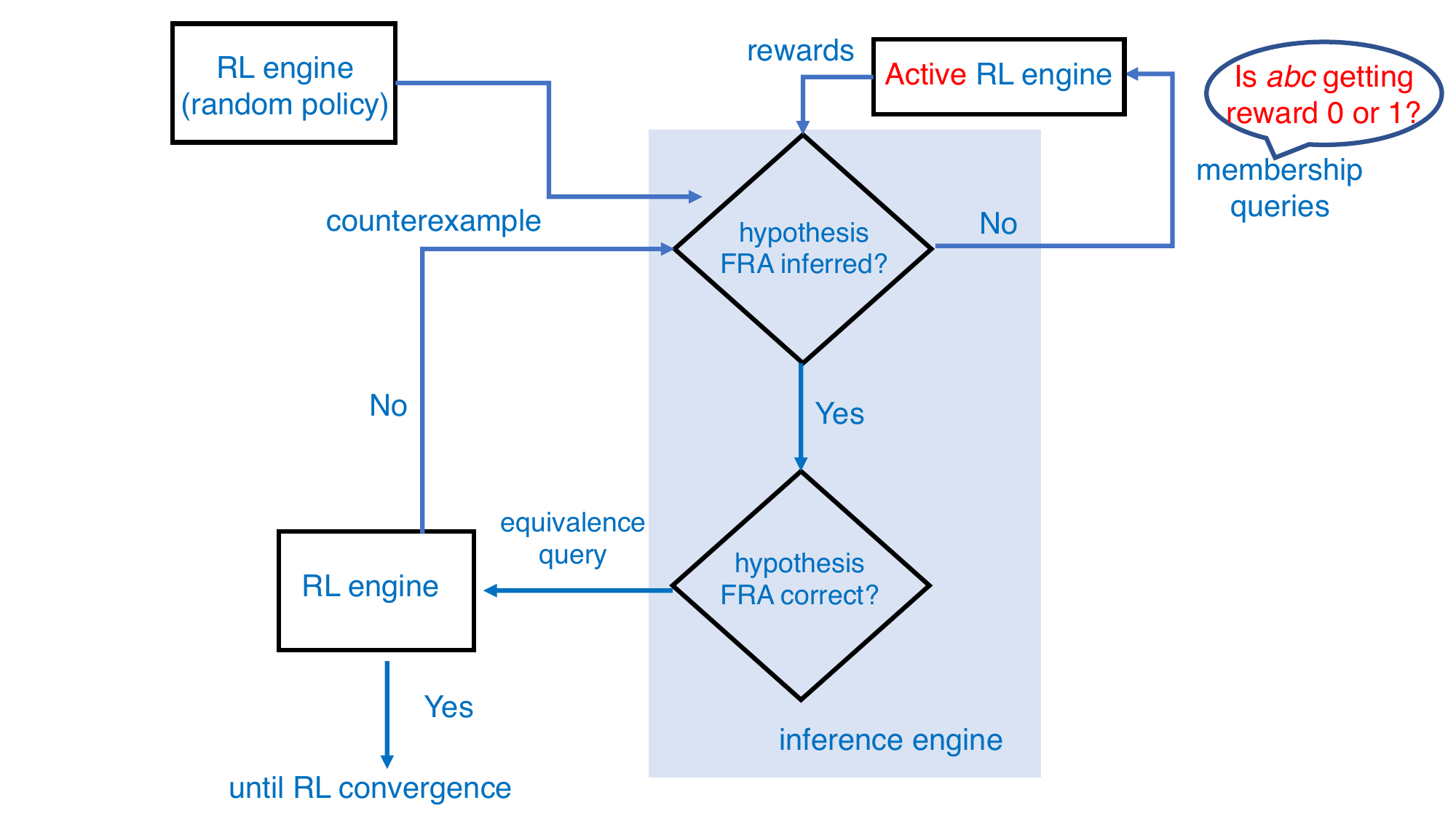}
		\caption{Block diagram of the \newAlgoName\ approach.
		Initially, a random RL engine is used generate counterexample traces for the inference engine. The inference engine alternately performs two tasks: (a) it creates membership queries for the RL engine if there are counterexamples that are inconsistent with the inferred Finite Reward Automaton (FRA); (b) it generates and sends equivalence queries to the RL-Engine. The active RL Agent will obtain rewards from the environment to answer the membership queries and the equivalence queries.
		}
		\label{block}
	\end{figure}

\subsection{Active Finite Reward Automaton Inference Engine}
In this subsection, we introduce the active finite reward automaton inference engine which is based on L* learning \cite{angluin1987learning}, which is an algorithm that learns a \emph{minimal} deterministic finite automaton (DFA) that accepts an unknown regular language $\mathcal{L}$ by interacting with a \emph{teacher}. In the setting of this paper, the role of teacher is fulfilled by the RL engine. For simplicity, we only consider RL tasks with non-Markovian reward functions with finitely many reward values (hence there exists at least one finite reward automaton that can encode the non-Markovian reward function).

We first show that a finite reward automaton can be converted to a DFA.

\begin{definition} 
A \emph{deterministic finite automaton (DFA)} is a five-tuple $\dfa = (\dfaStates, \dfaState_I, \dfaInputAlphabet, \dfaTransition, \dfaFinalStates)$ consisting of a nonempty, finite set $\dfaState$ of states, an initial state $\dfaState_I \in \dfaStates$, an input alphabet $\dfaInputAlphabet$,a transition function $\dfaTransition \colon \dfaStates \times \dfaInputAlphabet \to \dfaStates$, and a set $\dfaFinalStates \subseteq \dfaStates$ of final states.The size of an DFA, denoted by $|\dfa|$, is the number $|\dfaStates|$ of its states.
\end{definition}

A \emph{run} of an DFA $\dfa = (\dfaStates, \dfaState_I, \dfaInputAlphabet, \dfaTransition, \dfaFinalStates)$ on an input word $\tau = \tau_1 \ldots \tau_n$ is a sequence $\dfaState_0 \ldots \dfaState_n$ of states such that $\dfaState_0 = \dfaState_I$ and $\dfaState_i = \dfaTransition(\dfaState_{ix}, a_i)$ for each $i \in \{ 1, \ldots, n \}$.
A run $\dfaState_0 \ldots \dfaState_n$ of $\dfa$ on a word $u$ is \emph{accepting} if $\dfaState_n \in \dfaFinalStates$, and a word $u$ is \emph{accepted} if there exists an accepting run.
The \emph{language} of an DFA $\dfa$ is the set $L(\dfa) = \{ u \in \dfaInputAlphabet^\ast \mid \text{$\dfa$ accepts $u$} \}$.
As usual, we call two DFAs $\dfa_1$ and $\dfa_2$ \emph{equivalent} if $L(\dfa_1) = L(\dfa_2)$.

We show that every finite reward automaton over the input alphabet $\mealyInputAlphabet$ and output alphabet $\mealyOutputAlphabet$ can be translated into an ``equivalent'' DFA as defined below. This DFA operates over the combined alphabet $\rmInputAlphabet \times \rmOutputAlphabet$ and accepts a word $(\mdpLabel_0, \mdpRewards_0) \ldots (\mdpLabel_k, \mdpRewards_k)$ if and only if $\machine$ outputs the reward sequence $\mdpRewards_0 \ldots \mdpRewards_k$ on reading the label sequence $\mdpLabel_0 \ldots \mdpLabel_k$.

\begin{lemma} \label{lem:Mealy-to-DFA}
Given a finite reward automaton $\machine = (\mealyStates, \mealyInit, \mealyInputAlphabet, \mealyOutputAlphabet, \mealyTransition, \mealyOutput)$, one can construct a DFA $\dfa_\machine$ with $|\machine| + 1$ states such that
\begin{multline*}
\mathcal{L}(\dfa_\machine) = \bigl \{ (\mdpLabel_0, \mdpRewards_0) \ldots (\mdpLabel_k, \mdpRewards_k) \in (\rmInputAlphabet \times \rmOutputAlphabet)^\ast \mid \machine[\mdpLabel_0 \ldots \mdpLabel_k] = \mdpRewards_0 \ldots \mdpRewards_k \bigr \}.
\end{multline*}
\end{lemma}

\begin{proof}
Let $\machine = (\mealyStates_\machine, \mealyInit^\machine, \mealyInputAlphabet, \mealyOutputAlphabet, \mealyTransition_\machine, \mealyOutput_\machine)$ be a finite reward automaton.
Then, we define a DFA $\dfa_\machine = (\dfaStates, \dfaState_\init, \dfaInputAlphabet, \dfaTransition, \dfaFinalStates)$ over the combined alphabet $\rmInputAlphabet \times \rmOutputAlphabet$ by
\begin{itemize}
	\item $\dfaStates = \mealyStates_\machine \cup \{ \bot \}$ with $\bot \notin \mealyStates_\machine$;
	\item $\dfaState_\init = \mealyInit^\machine$;
	\item $\Sigma = \mealyInputAlphabet \times \mealyOutputAlphabet$;
	\item $\dfaTransition \bigl( \mealyState, (\mdpLabel, \mdpRewards) \bigr) = \begin{cases} \mealyState' & \text{if $\mealyTransition_\machine(\mealyState, \mdpLabel) = \mealyState'$ and $\mealyOutput_\machine(\mealyState, \mdpLabel) = \mdpRewards$;} \\ \bot & \text{otherwise;} \end{cases}$
	\item $\dfaFinalStates = \mealyStates_\machine$.
\end{itemize}

In this definition, $\bot$ is a new sink state to which $\dfa_\machine$ moves if its input does not correspond to a valid input-output pair produced by $\machine$.
A straightforward induction over the length of inputs to $\dfa_\machine$ shows that it indeed accepts the desired language.
In total, $\dfa_\machine$ has $|\machine| + 1$ states.
\end{proof}

During the learning process, the inference engine maintains an \emph{observation table} $\mathcal{O}=(S,E,T)$ where $S\subseteq\Sigma^*$ is a set of prefixes ($\Sigma=2^{\mathcal{P}}$), $E\subseteq\Sigma^*$ is a set of suffixes and $T:(S\cup S\cdot\Sigma)\times E\rightarrow\{0,1\} $. $\Sigma^*$ denotes finite traces from alphabet set $\Sigma$. For $s\in S\cup S\cdot\Sigma $, $e\in E$, if $s\cdot e\in\mathcal{L}$, then $T(s,e)=1$ and if $s\cdot e\notin \mathcal{L}$ then $T(s,e)=0$. Membership queries assign the correct value ($0$ or $1$) to $T(s,e)$. For simplicity, we denote $row(s)=(T(s,e_1),...,T(s,e_n))\in \{0,1\}^n,|E|=n$. The inference engine will always keep the observation table \emph{closed and consistent} as defined below.

	\begin{definition}
	An observation table $\mathcal{O}=(S,E,T)$ is closed if for each $t\in S\cdot\Sigma$, we can find some $s\in S$ such that $row(s)=row(t)$.
	\end{definition}
	\begin{definition}
		$\mathcal{O}$ is consistent if whenever for $s_1,s_2\in S$, $row(s_1)=row(s_2)$, then for any $\sigma\in\Sigma$, we have $row(s_1\sigma)=row(s_2\sigma)$.
	\end{definition}
	\begin{remark}
	\label{obsTableToDFA}
	If an observation table $\mathcal{O}=(S,E,T)$ is closed and consistent, it is possible to construct a DFA $\mathcal{M}(\mathcal{O})=(Q,\Sigma,\delta,Q_0,F)$ as the acceptor, where
	\begin{itemize}
		\item $Q=\{row(s)|s\in S\}$;
		\item $q_0=row(\epsilon)$;
		\item $\delta(row(s),\sigma)=row(s\sigma),\forall \sigma\in\Sigma$;
		\item $F=\{row(s)|s\in S, T(s)=1\}$.
	\end{itemize}
	\end{remark}

\begin{definition}
For a closed and consistent observation table $\mathcal{O}=(S,E,T)$, we define a corresponding hypothesis finite reward automaton $\machine^{\textrm{h}}(\mathcal{O}) = (\mealyStates^{\textrm{h}}, \mealyState^{\textrm{h}}_{init}, \rmInputAlphabet, \rmOutputAlphabet, \mealyTransition^{\textrm{h}}, \mealyOutput^{\textrm{h}})$ as follows:
\begin{itemize}
    \item $\mealyStates^{\textrm{h}}=\{row(s): s\in S\}$,
    \item $\mealyState^{\textrm{h}}_{init}=row(\epsilon)$,
    \item $\delta^{\textrm{h}}(row(s),\sigma)=row(s\sigma)$,
    \item $\mealyOutput^{\textrm{h}}(\mealyState, \inputTrace)=1, \mbox{if} ~T(\inputTrace)=1; \mbox{and}~ \mealyOutput^{\textrm{h}}(\mealyState, \inputTrace)=0, ~\mbox{otherwise}$.
\end{itemize}
\end{definition}

\begin{algorithm}[H]
\caption{\newAlgoName}
\label{whole_alg}
\begin{algorithmic}[1]
    \State Initialize $\mathcal{O}=(S, E, T)$, $sample$, $Nsample$, $C$, $\qValue_{\textrm{m}}$, $\qValue_{\textrm{h}}$
	\While{there exists counterexample $(\lambda, \rho)$}
		\State Add $(\lambda, \rho)$ and its prefixes to $S$ \label{add_Obs}
	    \State $ChangeT \gets 0$	
    	\While{$\mathcal{O}$ is neither closed nor consistent $\wedge$ $(ChangeT=0)$} \label{table_cc} 
		    \State$\chi\gets CheckObsTable(\mathcal{O})$\label{check_consist_close}
            \State $T, sample, Nsample, \qValue_{\textrm{m}}\gets  MQuery(T,\chi,sample,Nsample,C,\qValue_{\textrm{m}})$ \label{mem_query0}
        \EndWhile
    \State $(\lambda, \rho),\mathcal{O},sample,Nsample,\qValue_{\textrm{h}}\gets EQuery(\mathcal{O},sample,Nsample,\qValue_{\textrm{h}})$ \label{equ_query0}
    \EndWhile
\end{algorithmic}
\end{algorithm}

Algorithm \ref{whole_alg} shows the AFRAI-RL algorithm. Algorithm \ref{whole_alg} starts by checking whether there are any counterexample traces that need to be added to the observation table. These traces and their suffixes are added to the observation table (Algorithm \ref{whole_alg}, Line \ref{add_Obs}), and the subroutine $CheckObsTable$ (see Algorithm \ref{checkobstable}) is used to find \textit{membership query traces} (noted as $\chi$ in all algorithms).

Then, Algorithm \ref{whole_alg} proceeds to answer two types of queries, namely the \emph{membership} query (Algorithm \ref{whole_alg}, Line \ref{mem_query0}) and the \emph{equivalence} query (Algorithm \ref{whole_alg}, Line \ref{equ_query0}).

\begin{algorithm}[h]
	\caption{MQuery}
	\label{MembershipQuery}
	\begin{algorithmic}[1]
	\State \textbf{Input}: $T, \chi, sample, Nsample, C, \qValue_{\textrm{m}}$ 
    \State $Query\gets\textrm{membership}$\label{mem_query1}      
    \For{each $\zeta\in \chi$}
        \State Construct a query finite reward automaton $\machine^{\textrm{m}}(\zeta)$
        \State $counter\gets0$
        \State $PrevAnswer\gets check(\zeta,sample)$ \label{check}
        \If{$PrevAnswer\neq Null$} 
        \State $T(\zeta)=PrevAnswer$ \label{tableT}
        \Else
        \State $Inconsistent\gets 0$
            \While{$check(\zeta, sample)=Null\wedge (counter$ $<C)\wedge Inconsistent = 0$} \label{clock}
                \State $\inputTrace, \outputTrace, \qValue_{\textrm{m}}\gets$ RL-Engine($Query$, $\machine^{\textrm{m}}(\zeta)$, $\qValue_{\textrm{m}}$)
                \If{$\zeta$ is inconsistent with $(\inputTrace, \outputTrace)$} 
                    \State $Inconsistent \gets 1$ \label{inconsistent_flag}
                \Else
                    \State Add $(\inputTrace, \outputTrace)$ to $sample$ \label{updateSample}
                    \State $ChangeT,T\gets$CheckNSample($(\inputTrace, \outputTrace)$,$Nsample$,$T$)\label{checkN1}
                    \If{$ChangeT=1$}
                        \State \textbf{return} $T,sample,Nsample,\qValue_{\textrm{m}}$ \label{early_exit}
                    \EndIf
                    \State $counter\gets counter+1$ \label{check2}
                \EndIf \label{inconsistent2}
            \EndWhile
        \If{$(counter > C)\vee (Inconsistent = 1)$} \label{took_too_long}
            \State $T(\zeta)\gets0$ \label{inconsistent1}
            \State Add $\zeta$ to $Nsample$ \label{Nsample1}
        \Else
            \State $T(\zeta) \gets 1$ 
        \EndIf 
        \label{Nsample2}
    \EndIf
    \EndFor 
    \label{endFor}
    \State \textbf{return} $T,sample,Nsample,\qValue_{\textrm{m}}$
\end{algorithmic} 
\end{algorithm}

\noindent\textbf{Answering Membership Queries}: The detailed procedures to answer membership queries are shown in Algorithm \ref{MembershipQuery}. The subroutine sets $T(\zeta)$ to 1 if the trace $\zeta$ can be accepted by the DFA converted from a finite reward automaton (that can encode an unknown non-Markovian reward function), and sets $T(\zeta)$ to 0 otherwise. We maintain a set $sample$ of accepted traces and use the subroutine $CheckSample$ (see Algorithm \ref{CheckSample}) to check whether a membership query can already be answered by the set $sample$ (Algorithm \ref{MembershipQuery}, Lines \ref{check} to \ref{tableT}). If it can, we provide the answer to $T(\zeta)$ in the observation table $\mathcal{O}$ (Algorithm \ref{MembershipQuery}, Line \ref{tableT}); otherwise, we perform RL to answer the membership query (see Algorithm \ref{RLepisode} for details). If the membership query trace $\zeta$ is \textit{inconsistent} with a trace in RL from the environment (e.g., the membership query trace $\zeta=(\mdpLabel_1,0),(\mdpLabel_2,1),(\mdpLabel_3,1)$, while a trace $(\mdpLabel_1,0),(\mdpLabel_2,0)$ is observed from the environment), then a flag is set (Algorithm \ref{MembershipQuery}, Line \ref{inconsistent_flag}) to stop the loop, $T(\zeta)$ is set to zero immediately (Algorithm \ref{MembershipQuery}, Line \ref{inconsistent1}), and move on to the next trace. Otherwise we add the trace to $sample$ (Algorithm \ref{MembershipQuery}, Line \ref{updateSample}).
To boost efficiency, we set a limit on how many episodes we perform to answer a membership query. This limit is $C\in\mathbb{Z}_{>0}$. We answer the membership query as 0 if after $C$ episodes the membership query still cannot be answered (Algorithm \ref{MembershipQuery}, Lines \ref{took_too_long}-\ref{inconsistent1}). Such traces are recorded in the set $Nsample$ (Algorithm \ref{MembershipQuery}, Line \ref{Nsample1}). Afterwards, if the trace for the membership query is encountered in the environment, we use the subroutine $CheckNSample$ (see Algorithm \ref{CheckNSample}) to change the original answers in the observation table $\mathcal{O}$ accordingly. This is performed during both membership (Algorithm \ref{MembershipQuery}, Line \ref{checkN1}) and equivalence queries (Algorithm \ref{Equivalence_Query}, Line \ref{checkN2}). If the answer was changed during a membership query(i.e., $ChangeT = 1$), we exit the Algorithm \ref{MembershipQuery} (Line \ref{early_exit}) to generate the additional membership query traces created from changing the table (Algorithm \ref{whole_alg}, Line \ref{table_cc}).


\begin{algorithm}[H]
\caption{EQuery}
\label{Equivalence_Query}
\begin{algorithmic}[1]
\State \textbf{Input} $\mathcal{O}, sample, Nsample, \qValue_{\textrm{h}}$
\State Construct a hypothesis finite reward automaton $\machine^{\textrm{h}}(\mathcal{O})$ 
  	\State $Query\gets\textrm{equivalence}$
    \State \textbf{Do} 
        \State $\inputTrace, \outputTrace, \qValue_{\textrm{h}}\gets$ RL-Engine($Query$,$\machine^{\textrm{h}}$, $\qValue_{\textrm{h}}$)\label{rl_return} 
        \State Add $(\inputTrace, \outputTrace)$ to $sample$ \label{up_sam}
        \State $ChangeT,T\gets$CheckNSample($(\inputTrace,\outputTrace)$,$Nsample$,$T$)\label{checkN2}
    \State \textbf{Until} Find counterexample $(\lambda, \rho)$ \label{equ_query1}
    \State \textbf{Return} $(\lambda, \rho),\mathcal{O},sample,Nsample,\qValue_{\textrm{h}}$ \label{to_inference}
\end{algorithmic}
\end{algorithm}

\noindent\textbf{Answering Equivalence Queries}: The detailed procedures to answer equivalence queries are shown in Algorithm \ref{Equivalence_Query}. We perform RL with the hypothesis finite reward automaton, updating $sample$ along the way (Algorithm \ref{Equivalence_Query}, Line \ref{up_sam}), until a \textit{counterexample} is found (Algorithm \ref{Equivalence_Query}, Line \ref{equ_query1}). A counterexample is a trace where the rewards given by environment are different from the rewards given by hypothesis finite reward automaton.
Specifically, there are two types of counterexamples. A \emph{positive counterexample} is a trace that is accepted by the DFA converted from the current hypothesis finite reward automaton, but is not accepted by the DFA converted from any finite reward automaton that can encode the unknown non-Markovian reward function. A \emph{negative counterexample} is a trace that is not accepted by the DFA converted from the current hypothesis finite reward automaton, but is accepted by the DFA converted from any finite reward automaton that can encode the unknown non-Markovian reward function. The RL-engine returns counterexamples to the inference engine for another round of inference (Algorithm \ref{Equivalence_Query}, Line \ref{rl_return} and \ref{to_inference}). 

\begin{algorithm}[h]
    \caption{CheckObsTable}
    \label{checkobstable}
    \begin{algorithmic}[1]
        \State \textbf{Input}: $\mathcal{O}$ 
        \If{$\mathcal{O}$ is not consistent}
            \State Find $s_1$, $s_2\in S$, $\sigma\in \Sigma$ and $e\in E$ such that $row(s_1)=row(s_2)$ and
            $T(s_1\sigma e)\neq T(s_2\sigma e)$
            \State add $\sigma e$ to $E$
            \State $\chi\gets(S\cup S\Sigma)\sigma e$
    	\ElsIf{$\mathcal{O}$ is not closed}
            \State Find $s\in S$ and $\sigma\in\Sigma$ such that $\forall s\in S, row(s\sigma)\neq row(s)$
            \State add $s\sigma$ to $S$  
            \State $\chi\gets(s\sigma\cup s\sigma\Sigma)E$
        \EndIf
        \State Return $\chi$
    \end{algorithmic}
    
\end{algorithm}
\begin{algorithm}[h]
	\caption{CheckSample}                              
	\label{CheckSample}
	\begin{algorithmic}[1]
		\State \textbf{Input}: $\zeta$, $sample$ 
		 \For{each trace $(\lambda, \rho)$ in $sample$}
		 \If{$\zeta$ is prefix of $(\lambda, \rho)$}
            \State Return 1
        \EndIf
        \If{$\zeta$ is inconsistent with $(\lambda, \rho)$}
            \State Return 0
        \EndIf
        \EndFor
         \State Return $Null$
	\end{algorithmic} 
\end{algorithm}		
\begin{algorithm}[h]
	\caption{CheckNSample}                              
	\label{CheckNSample}
	\begin{algorithmic}[1]
		\State \textbf{Input}: $(\lambda, \rho)$, $Nsample$, $T$ 
		\For{each $\zeta'\in Nsample$} \label{checkNsample1}
            \If{$\zeta'$ is a prefix of ($\inputTrace, \outputTrace$)}
            \State $ChangeT=1$, $T(\zeta')=1$
            \EndIf
        \EndFor \label{checkNsample2}
        \State Return $ChangeT$,  $T$
	\end{algorithmic} 
\end{algorithm}	

Algorithm \ref{checkobstable} generates \textit{membership query traces} based on whether or not the observation table $\mathcal{O}$ is closed or consistent. If the table is not consistent, then we add $\sigma e$ to $E$ and each $\zeta\in (S\cup S\Sigma)\sigma e$ forms a \textit{membership query trace}. If the table is not closed, then we add add $s\sigma$ to $S$ and each $\zeta\in (s\sigma\cup s\sigma\Sigma)E$ forms a \textit{membership query trace}.

 Algorithm \ref{CheckSample} shows the subroutine $CheckSample$. It returns 1 or 0 if the membership query for $\zeta$ has already been answered, and $Null$ otherwise. For each trace $(\lambda, \rho)$ in $sample$, if a membership query trace $\zeta$ is prefix of $(\lambda, \rho)$, then $\zeta$ must be accepted by the DFA converted from the finite reward automaton that encodes the unknown non-Markovian reward function; hence $CheckSample$ returns 1. If $\zeta$ is \textit{inconsistent} with a trace in $sample$, then $\zeta$ cannot be accepted by the DFA mentioned above and $CheckSample$ returns 0.
 
 Algorithm \ref{CheckNSample} shows the subroutine $CheckNSample$. Each trace $\zeta'$ in $Nsample$ is checked to see if it is a prefix of $(\lambda, \rho)$ (the recent answer from the RL-Engine). If the trace is a prefix, then its answer in the observation table is changed (Line 4). A flag, $ChangeT$, is set so that the observation table is rechecked for being closed and consistent in Algorithm \ref{whole_alg}. 

By answering the membership and equivalence queries, $L^*$ algorithm is guaranteed to converge to the minimum DFA accepting the unknown regular language $\mathcal{L}$   using $O(|\Sigma|n^2+n\log c)$ membership queries and at most $n-1$ equivalence queries, where $n$ denotes the number of states in the final DFA and $c$ is the length of the longest counterexample from the RL engine when answering equivalence queries \cite{angluin1987learning}. 

\subsection{Active Reinforcement Learning Engine}
In this subsection, we introduce the active reinforcement learning engine.
We first define a \textit{query finite reward automaton} corresponding to a membership query trace $\zeta=(\mdpLabel_1,\mdpRewards_1),\ldots,(\mdpLabel_k,\mdpRewards_k)$ as follows.
\begin{definition}
	For a membership query trace $\zeta=(\mdpLabel_1,\mdpRewards_1),\ldots,(\mdpLabel_k,\mdpRewards_k)$, we define a corresponding query finite reward automaton $\machine^{\textrm{m}}(\zeta) = (\mealyStates^{\textrm{m}}, \mealyState^{\textrm{m}}_{init}, \rmInputAlphabet, \rmOutputAlphabet, \mealyTransition^{\textrm{m}}, \mealyOutput^{\textrm{m}})$ as follows:
	\begin{itemize}
		\item $\mealyStates^{\textrm{m}}=\{\mealyState^{\textrm{m}}_0, \mealyState^{\textrm{m}}_1, \dots, \mealyState^{\textrm{m}}_k\}$,
		\item $\mealyState^{\textrm{m}}_{init}=\mealyState^{\textrm{m}}_0$,
		\item for any $i\in[0, k-1], \delta^{\textrm{m}}(\mealyState^{\textrm{m}}_i,\mdpLabel_{i+1})=\mealyState^{\textrm{m}}_{i+1}$,  $\delta^{\textrm{m}}(\mealyState^{\textrm{m}}_i,\mdpLabel)=\mealyState^{\textrm{m}}_{i}$ for any $\mdpLabel\neq\mdpLabel_i$, and
		\item for any $\mealyState^{\textrm{m}}\in\mealyStates^{\textrm{m}}$ and any $\mdpLabel\in\mealyInputAlphabet$, $\mealyOutput^{\textrm{m}}(\mealyState^{\textrm{m}}, \mdpLabel)=1, \mbox{if} ~\delta^{\textrm{m}}(\mealyState^{\textrm{m}},\mdpLabel)\neq\mealyState^{\textrm{m}}; \mbox{and}~ \mealyOutput^{\textrm{m}}(\mealyState^{\textrm{m}}, \mdpLabel)=0, ~\mbox{otherwise}$.
	\end{itemize}
\end{definition}

Intuitively, the query finite reward automaton corresponding to a membership query trace $\zeta=(\mdpLabel_1,\mdpRewards_1),\ldots,(\mdpLabel_k,\mdpRewards_k)$ is a finite reward automaton that outputs a reward of one every time a new label $\mdpLabel_i$ ($i\in[1,k]$) is achieved (and the state of the finite reward automaton moves from $\mealyState^{\textrm{m}}_i$ to $\mealyState^{\textrm{m}}_{i+1}$). Therefore, in performing RL with the query finite reward automaton, the rewards obtained from the query finite reward automaton serve as incentives to encounter the label sequence in the membership query trace and hence to answer the membership query.

In the RL engine, we maintain two different types of q-functions: \textit{query q-functions} for answering membership queries, denoted as $q_{\textrm{m}}$; and \textit{hypothesis q-functions} for maximizing the cumulative rewards from the environment (also answering equivalence queries), denoted as $q_{\textrm{h}}$. 
 	
We update the query q-functions as follows:
	\begin{align}
	\begin{split}
	\qValue_{\textrm{m}}(\mdpState, \mealyState^{\textrm{m}}, \mdpAction)\gets&(1-\alpha)\qValue_{\textrm{m}}(\mdpState, \mealyState^{\textrm{m}}, \mdpAction)\\& +\alpha(\mdpRewards_{\textrm{m}}+\gamma\max\limits_{\mdpAction}\qValue_{\textrm{m}}(\mdpState', \mealyState'^{\textrm{m}}, \mdpAction)),
	\end{split}
	\label{query_q}
	\end{align}
	
Similarly, we update the hypothesis q-functions as follows:
		\begin{align}
	\begin{split}
	\qValue_{\textrm{h}}(\mdpState, \mealyState^{\textrm{h}}, \mdpAction)\gets&(1-\alpha)\qValue_{\textrm{h}}(\mdpState, \mealyState^{\textrm{h}}, \mdpAction)\\& +\alpha(\mdpRewards_{\textrm{h}}+\gamma\max\limits_{\mdpAction}\qValue_{\textrm{h}}(\mdpState', \mealyState'^{\textrm{h}}, \mdpAction)),
	\end{split}
	\label{hypothesis_q}
    \end{align}



\begin{algorithm}[ht]
	\caption{RL-Engine}
	\label{RLepisode}
	\begin{algorithmic}[1]
	    \State \textbf{Hyperparameters:} learning rate $\alpha$, discount factor $\gamma$, episode length $\mathit{eplength}$ 
		\State \textbf{Input:} Variable $Query$, a reward automaton $\machine$,
		 q-function $\qValue$
		\State $\mdpState \gets \mathit{InitialState()}; (\inputTrace, \outputTrace) \gets []$ 
		\State $\mealyState \gets \mealyState_{init}$
		\For {$0 \leq t < \mathit{eplength}$} \label{algLine:taskLoopStart}
		   \State $\mdpState'$,$\mealyState'$,$\qValue\gets Step(Query,\machine,\qValue,\mdpState,\mealyState)$\label{w_init}
            \State append $(\rmLabelingFunction(\mdpState, \mdpAction, \mdpState'),\mdpRewards)$ to $(\inputTrace,\outputTrace)$
	        \State $\mdpState\gets\mdpState'$, $\mealyState\gets\mealyState'$, $t\gets t+1$

		\EndFor \label{algLine:taskLoopEnd}
	
		\State \Return $(\inputTrace, \outputTrace, \qValue)$\label{return_RL}
	\end{algorithmic}
\end{algorithm}
\begin{algorithm}[ht]
    \caption{Step}
    \label{step}
    \begin{algorithmic}[1]
        \State \textbf{Input:} Variable $Query$, a finite reward automaton $\machine$, a q-function $\qValue$, an MDP state $\mdpState$,and an FRA state $\mealyState$
        \State $\mdpAction = \text{GetEpsilonGreedyAction}(\qValue,\mealyState, \mdpState)$
		\label{algLine:choose_action}
		
		\State $\mdpState' = \text{ExecuteAction}(\mdpState, \mdpAction)$
		
		\If{$Query=\textrm{membership}$}
			\State $\mdpRewards\gets\mealyOutput(\mealyState, \rmLabelingFunction(\mdpState, \mdpAction, \mdpState'))$;
			\Else 
			\State Observe $\mdpRewards$ from the environment;
		\EndIf \label{choose_reward}
		
		\State $\mealyState' =  \mealyTransition( \mealyState, \rmLabelingFunction(\mdpState, \mdpAction, \mdpState') )$ \label{get_mealy_state}
		
		\State $\qValue(\mdpState,\mealyState, \mdpAction)\gets(1-\alpha)\qValue(\mdpState,\mealyState, \mdpAction)+\alpha(\mdpRewards+\gamma\max\limits_{\mdpAction}\qValue(\mdpState',\mealyState', \mdpAction))$
		\label{update_q}
		\For{$\hat{\mealyState} \in \mealyStates \setminus \set{w}$} \label{offlinelearningLoopStart}
			\State $\hat\mealyState' = \mealyTransition(\hat\mealyState, \rmLabelingFunction(\mdpState, \mdpAction, \mdpState'))$
            \State $\hat\mdpRewards = \mealyOutput(\hat\mealyState, \rmLabelingFunction(\mdpState, \mdpAction, \mdpState'))$ 
			\State $\qValue(\mdpState,\hat\mealyState, \mdpAction)\gets(1-\alpha)\qValue(\mdpState,\hat\mealyState, \mdpAction)+\alpha(\hat\mdpRewards+\gamma\max\limits_{\mdpAction}\qValue(\mdpState',\hat\mealyState', \mdpAction))$ 
		\EndFor \label{offlinelearningLoopEnd}
	    
	    \State \Return $\mdpState'$, $\mealyState'$, $\qValue$
    \end{algorithmic}
\end{algorithm}
Algorithm \ref{RLepisode} shows the procedure to run each RL episode to answer an membership query or equivalence query. Algorithm \ref{RLepisode} first initializes the initial state $x$ of the MDP, the initial state of the (query or hypothesis) finite state automaton $w$ and the trace (as the empty trace). Specifically, $w$ is initialized as $\mealyState_{init}^{\textrm{m}}$ if it is in the membership query phase ($\machine = \machine_{\textrm{m}}$), and initialized as $\mealyState_{init}^{\textrm{h}}$ if it is in the equivalence query phase ($\machine = \machine_{\textrm{h}}$) (Line \ref{w_init}). Algorithm \ref{step} is used to run one step through the environment (see next paragraph for details). We feed the query q-function and finite reward automaton into algorithm \ref{step} if a membership query is being asked, and the hypothesis q-function and automaton otherwise. Algorithm \ref{RLepisode} returns the trace, the query q-function and the hypothesis q-function (Line \ref{return_RL}).

Algorithm \ref{step} runs one step through the environment, updating the q-function for either a membership or equivalence query. First, at state $x$ an action $a$ is selected according to the q-function using the epsilon-greedy approach and executed to reach a new state $x'$ (Line \ref{algLine:choose_action}). Then rewards are collected based on the type of query being asked. If a membership query is being asked, then the automaton's reward function is used; otherwise the reward is observed from the environment. The next mealy state for the automaton is calculated (Line \ref{get_mealy_state}) and the q-function supplied is updated (Line \ref{update_q}).

\begin{lemma}
\label{episodesToLearn}
	 Let $\mdp$ be a labeled MDP and $\machine$ the finite reward automaton encoding the reward function of $\mdp$. Then, AFRAI with $eplength \ge 2^{|\mdp|+1}(|\machine| + 1) - 1$ almost surely learns a finite reward automaton in the limit that is equivalent to $\machine$ on all attainable label sequences.
	\end{lemma}
\begin{proof} 
	Given a $eplength = 2^{|\mdp|+1}(|\machine| + 1) - 1$, we are almost sure to experience every possible attainable trace from the environment. The proof of this is similar to the proof provided for Lemma 2 in \cite{Xu2019jirp} (Proof is located in Appendix C of \cite{Xu2019jirp}). Given every attainable trace, AFRAI-RL will answer all membership and equivalence queries correctly with probability 1 in the limit, because the observation table $\mathcal{O}$ can be changed as the algorithm \ref{whole_alg} runs (Algorithm \ref{whole_alg}, Lines \ref{mem_query0} and \ref{equ_query0}).  
	Furthermore, the author in \cite{angluin1987learning} shows that if all membership and equivalence queries can be answered, a DFA can be formed from the observation table (as in Remark \ref{obsTableToDFA}). This DFA will be the smallest DFA that can accept the language defined in Lemma \ref{lem:Mealy-to-DFA}. In this RL context, this language will match the language of the DFA converted from the finite reward automaton; i.e., the language will encode the finite reward automaton that is equivalent to $\machine$.
	Therefore, the observation table encodes the finite reward automaton that is equivalent to $\machine$.
	
\end{proof}
With Lemma \ref{episodesToLearn}, we can proceed to prove Theorem 3.
\begin{theorem}
\label{thm:convergenceInTheLimit}
Let $\mdp$ be a labeled MDP and $\machine$ the reward machine encoding of the reward function of $\mdp$. Then, \newAlgoName\ with $\mathit{eplength} \geq 2^{|\mdp| + 1} \cdot (|\machine| + 1) - 1$ almost surely converges to an optimal policy in the limit.
\end{theorem}
\begin{proof}
Lemma 2 shows that, eventually, the reward machine learned by \newAlgoName, will be equivalent to $\machine$ on all attainable label sequences. Let $\mathcal{H}$ be the reward machine learned by \newAlgoName~ and $\mdp_H$ be the product MDP.

Thus an optimal policy for $\mdp_H$ will also be optimal for $\mdp$. When running episodes of QRM (Algorithm \ref{RLepisode}) under the reward machine $\mathcal{H}$, an update of a q-function connected to a state of $\mathcal{H}$ corresponds to updating the q function for $\mdp_H$. Since $eplength
\ge |M|$, the fact that
QRM uses the epsilon-greedy strategy and that updates are done in parallel for all states of $\mathcal{H}$ implies that every state-action pair of the $\mdp_H$ will be seen infinitely often. Hence, the convergence of q-learning for $\mdp_H$ to an optimal policy is guaranteed by \cite{Watkins1992}. Therefore, as the number of episodes goes to infinity, with  $eplength \ge 2^{|\mdp| + 1} \cdot (|\machine| + 1) - 1$, AFRAI-RL converges towards an optimal policy.
\end{proof}

\section{Case Studies}
\label{sec_case}

In this section, we apply the proposed approach to the office world scenario adapted from~\cite{DBLP:conf/icml/IcarteKVM18} and the craft world scenario from~\cite{andreas2017modular}. We perform the following four different methods:
\begin{itemize}
	
	\item \textcolor{black}{\newAlgoName: We use the libalf~\cite{libalfTool} implementation of active automata learning as the algorithm to infer finite reward automata. }\\
	
	\item \textcolor{black}{\algoName{}-SAT: We use the libalf~\cite{libalfTool} implementation of SAT-solving (see Section 3.2 of \cite{Xu2019jirp}) as the algorithm to infer finite reward automata. }\\
	
	\item \textcolor{black}{LRM-QRM: We use the QRM implementation from ~\cite{DBLP:conf/icml/IcarteKVM18}, adapted to test the agent at the end of each episode.}\\
	
	\item \textcolor{black}{PPO2: We use the Stable Baselines implementation of PPO2 \cite{schulman2017proximal}. The state space is a history of the past states the agent has been in. This was added because PPO2 doesn't have any way of remembering it's previous states.}

\end{itemize}

In \cite{Xu2019jirp}, the authors have shown that JIRP-SAT and JIRP-RPNI outperform q-learning in augmented
state space (QAS), hierarchical reinforcement learning (HRL), and deep reinforcement learning with double
q-learning (DDQN) in three case studies. Therefore, if we can show that \newAlgoName~outperforms JIRP-SAT and JIRP-RPNI, then we can deduce that \newAlgoName~outperforms QAS, HRL and DDQN as well.

\subsection{\Office}
\begin{figure}[t]
	\centering
	\includegraphics[width=0.35\textwidth]{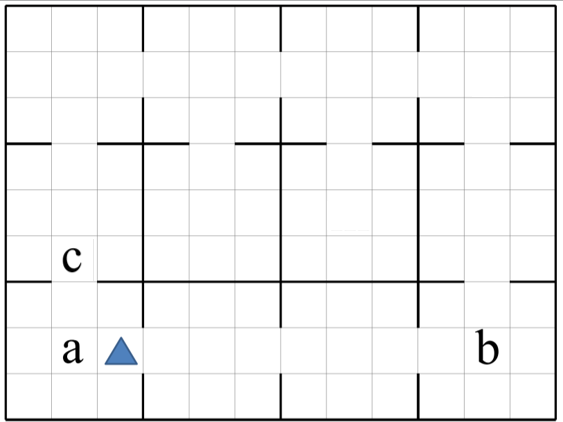}
	\caption{\textcolor{black}{The map in the \office.}}
	\label{office_map_m}
\end{figure}
We consider the \office\ in the 9$\times$12 grid-world. \textcolor{black} {Figure \ref{office_map_m} shows the map in the \office. We use the triangle to denote the initial position of the agent.} The agent has four possible actions at each time step: move north, move south, move east and move west. After each action, the robot may slip to each of the two adjacent cells with probability of 0.05. In Algorithm \ref{whole_alg}, we set $C=500$.

\textcolor{black} {We consider the following three tasks:\\
\textbf{\OfficeA}: first go to $\textrm{a}$, then go to $\textrm{b}$ and $\textrm{a}$ in this order, and finally return to $\textrm{c}$. Episode Length was set to 200, and total training time was set to 1,000,000.\\
\textbf{\OfficeB}: first go to $\textrm{b}$, then go to $\textrm{c}$ and $\textrm{a}$ in this order, then repeat the sequence. Episode length was set to 800, and total training time was set to 2,000,000.\\
\textbf{\OfficeC}: first go to $\textrm{c}$, then go to $\textrm{b}$ and $\textrm{a}$ in this order, then go to $\textrm{b}$ and $\textrm{c}$ in this order and return to $\textrm{a}$. Episode length was set to 800, and total training time was set to 6,000,000.}


Figure \ref{office_tasks} shows the attained rewards of 10 independent simulation runs for each task, averaged every 10 training steps. For task 1, it can be seen that, on average,
the proposed \newAlgoName\ approach converges to an optimal policy in about 0.2 million training steps, while \methodA-SAT, LRM-QRM, and PPO2 do not converge to an optimal policy. For task 2, on average the proposed \newAlgoName\ approach converges to an optimal policy in about 1.8 million training steps, while \methodA-SAT, LRM-QRM, and PPO2 do not converge to an optimal policy. For task 3, on average the proposed \newAlgoName\ approach converges to an optimal policy in about 4.0 million training steps, while \methodA-SAT converges to an optimal policy in about 4.5 million training steps and LRM-QRM, and PPO2 do not converge to an optimal policy.


\newcommand{\figurespacing}{0.32}
\begin{figure*}
	\centering
\begin{subfigure}[b]{\figurespacing\textwidth}  
	\centering
	\includegraphics[width=1.03\textwidth]{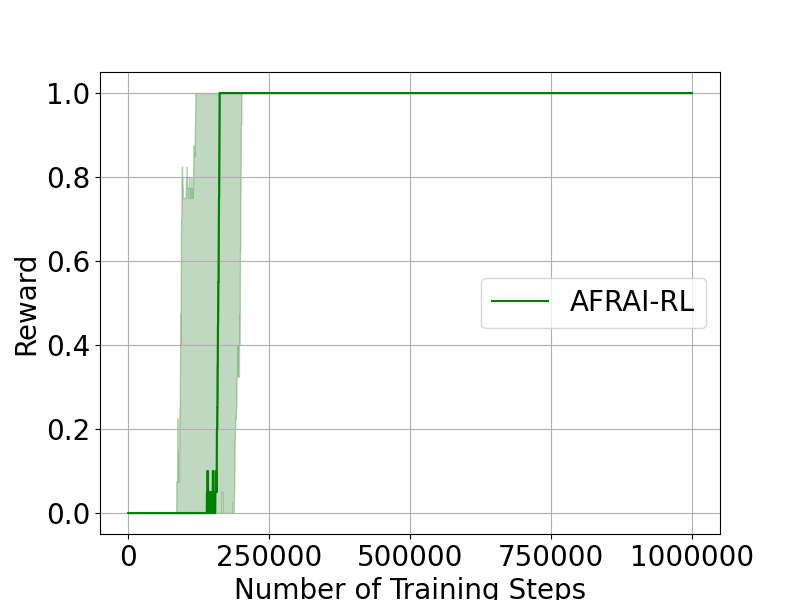}
\end{subfigure}
\begin{subfigure}[b]{\figurespacing\textwidth}  
	\centering
	\includegraphics[width=1.03\textwidth]{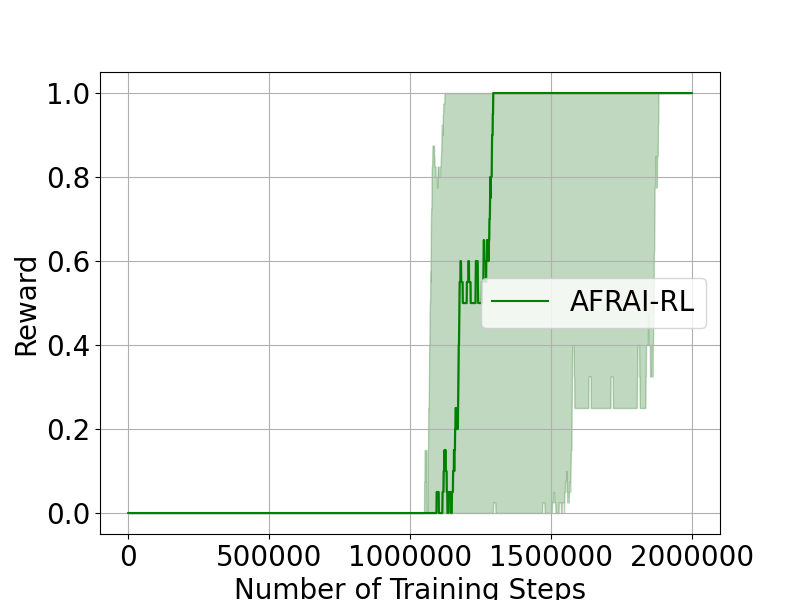}
	\end{subfigure}
\begin{subfigure}[b]{\figurespacing\textwidth}  
	\centering
	\includegraphics[width=1.03\textwidth]{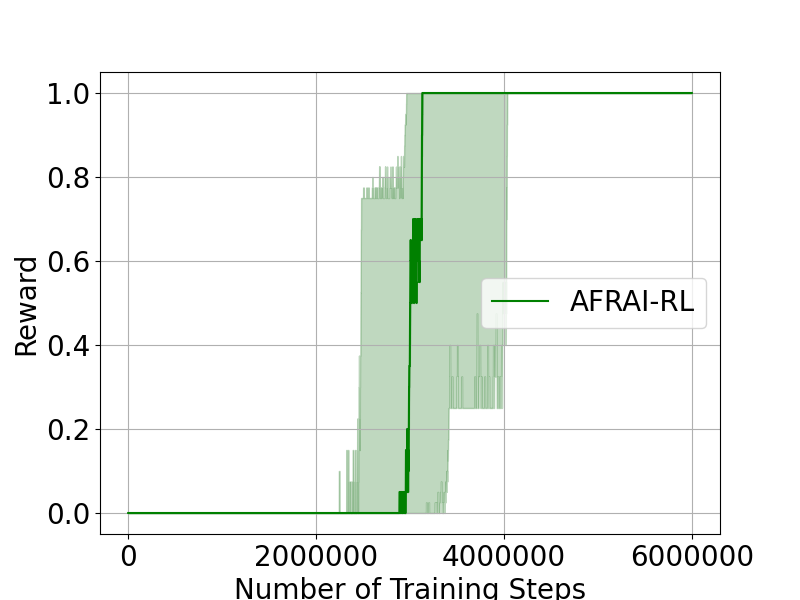}
	\end{subfigure}
\begin{subfigure}[b]{\figurespacing\textwidth}  
	\centering
	\includegraphics[width=\textwidth]{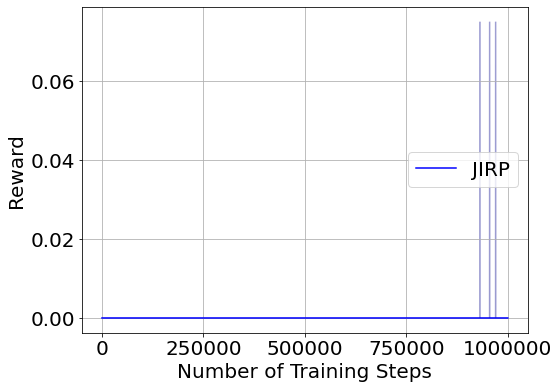}
	\end{subfigure}
\begin{subfigure}[b]{\figurespacing\textwidth}  
	\centering
	\includegraphics[width=\textwidth]{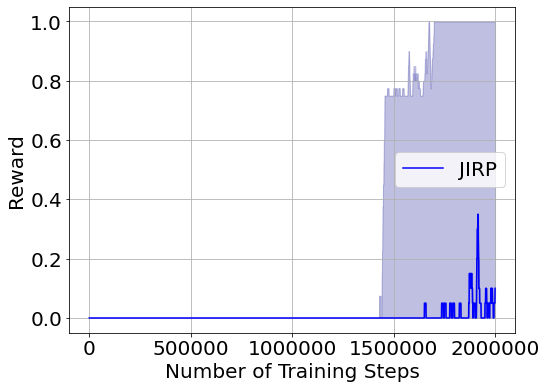}
\end{subfigure}
\begin{subfigure}[b]{\figurespacing\textwidth}  
	\centering
	\includegraphics[width=\textwidth]{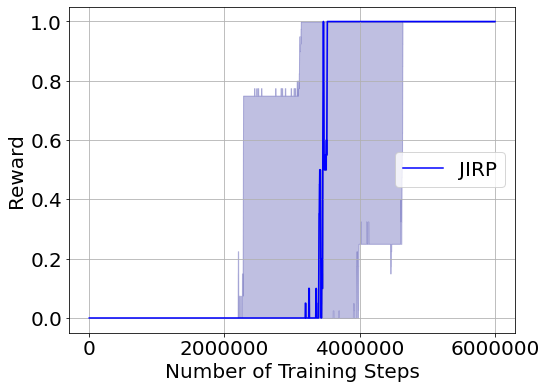}
	\end{subfigure}
\begin{subfigure}[b]{\figurespacing\textwidth}  
	\centering
	\includegraphics[width=0.96\textwidth]{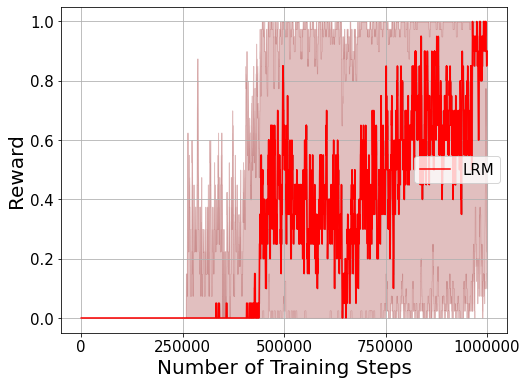}
	\end{subfigure}
\begin{subfigure}[b]{\figurespacing\textwidth}  
	\centering
	\includegraphics[width=0.96\textwidth]{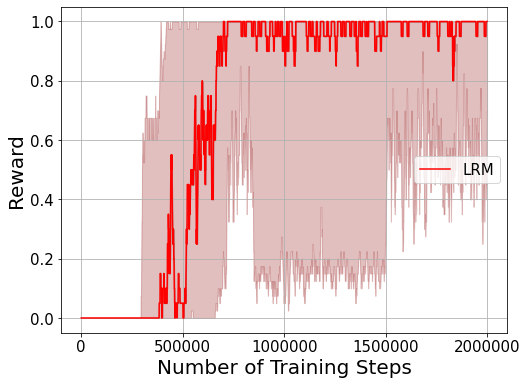}
	\end{subfigure}
\begin{subfigure}[b]{\figurespacing\textwidth}  
	\centering
	\includegraphics[width=0.96\textwidth]{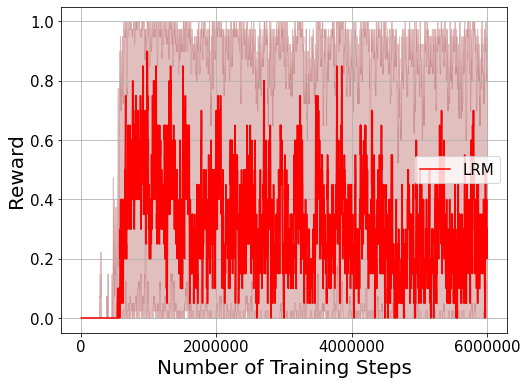}
\end{subfigure}
\begin{subfigure}[b]{\figurespacing\textwidth}  
	\centering
	\includegraphics[width=\textwidth]{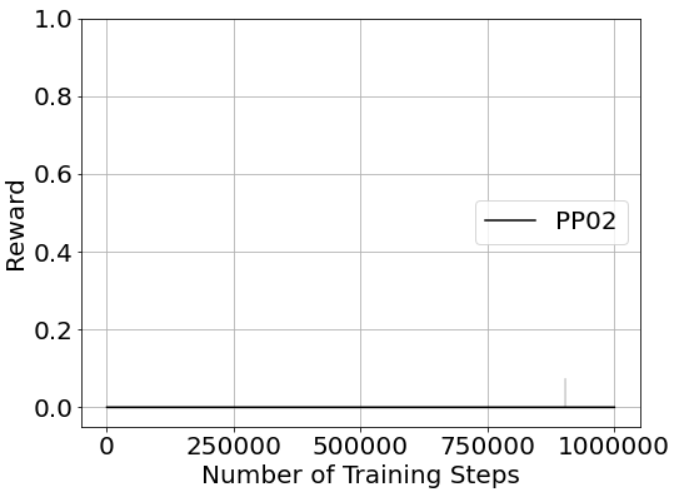}
	\caption{Task 1}
	\end{subfigure}
\begin{subfigure}[b]{\figurespacing\textwidth}  
	\centering
	\includegraphics[width=\textwidth]{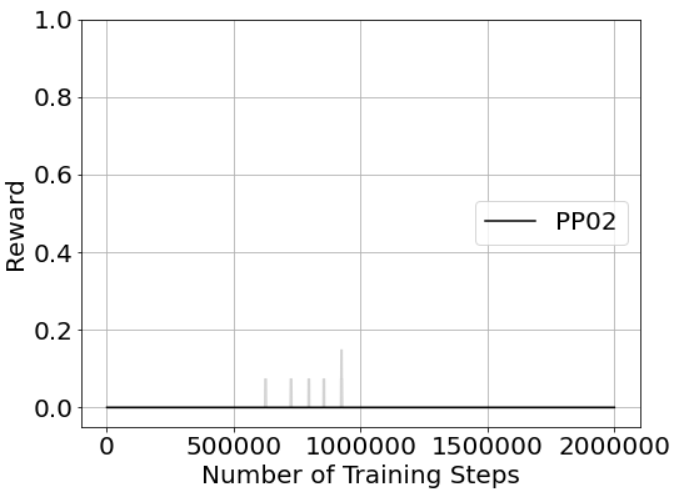}
	\caption{Task 2}
	\end{subfigure}
\begin{subfigure}[b]{\figurespacing\textwidth}  
	\centering
	\includegraphics[width=\textwidth]{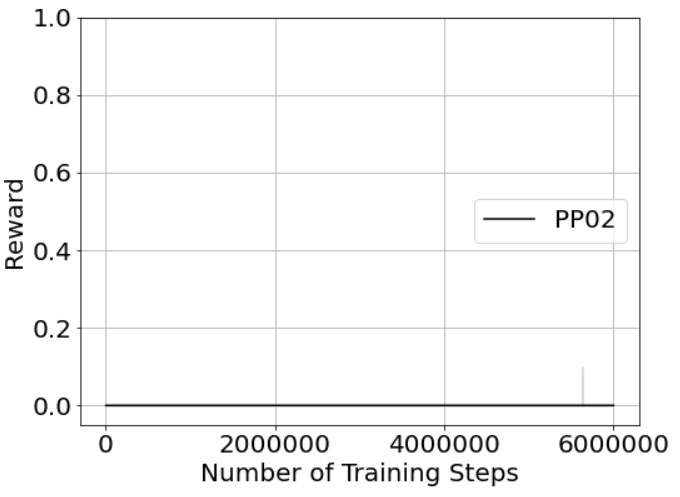}
	\caption{Task 3}
	\end{subfigure}
\caption{Attained rewards of 10 independent simulation runs of the \office, averaged for every 10 training steps in \newAlgoName~(First row), \methodA-SAT (Second Row), LRM-QRM (Third Row), and PPO2 (Fourth Row): (a) Task 1; (b) Task 2; (c) Task 3.}  
\label{office_tasks}
\end{figure*}
\newcommand{\figurespacingCraft}{0.35}

\subsection{\craft}
We consider the \craft~in the 21$\times$21 grid-world \cite{andreas2017modular}. The four actions and the slip rates are the same as in the \office. We train on two tasks: making a hammer and spear. In Algorithm \ref{whole_alg}, we set $C=500$.

We consider the following two tasks, noting the symbol used for each object in parenthesis:\\
\textbf{\OfficeA}: Build a Hammer. Agent must collect string ($\textrm{b}$), stone ($\textrm{e}$), iron ($\textrm{f}$), stone ($\textrm{e}$) (in this order) and travel to the workbench ($\textrm{c}$) to make the hammer. Episode length was set to 400, and total training time was set to 400,000.\\
\textbf{\OfficeB}: Build a Spear. The agent must collect string, stone, wood ($\textrm{a}$), string and travel to the workbench to make the spear. Episode length was set to 400, and total training time was set to 250,000.
\begin{figure*}
	\centering
\begin{subfigure}[b]{\figurespacingCraft\textwidth}  
	\centering
	\includegraphics[width=\textwidth]{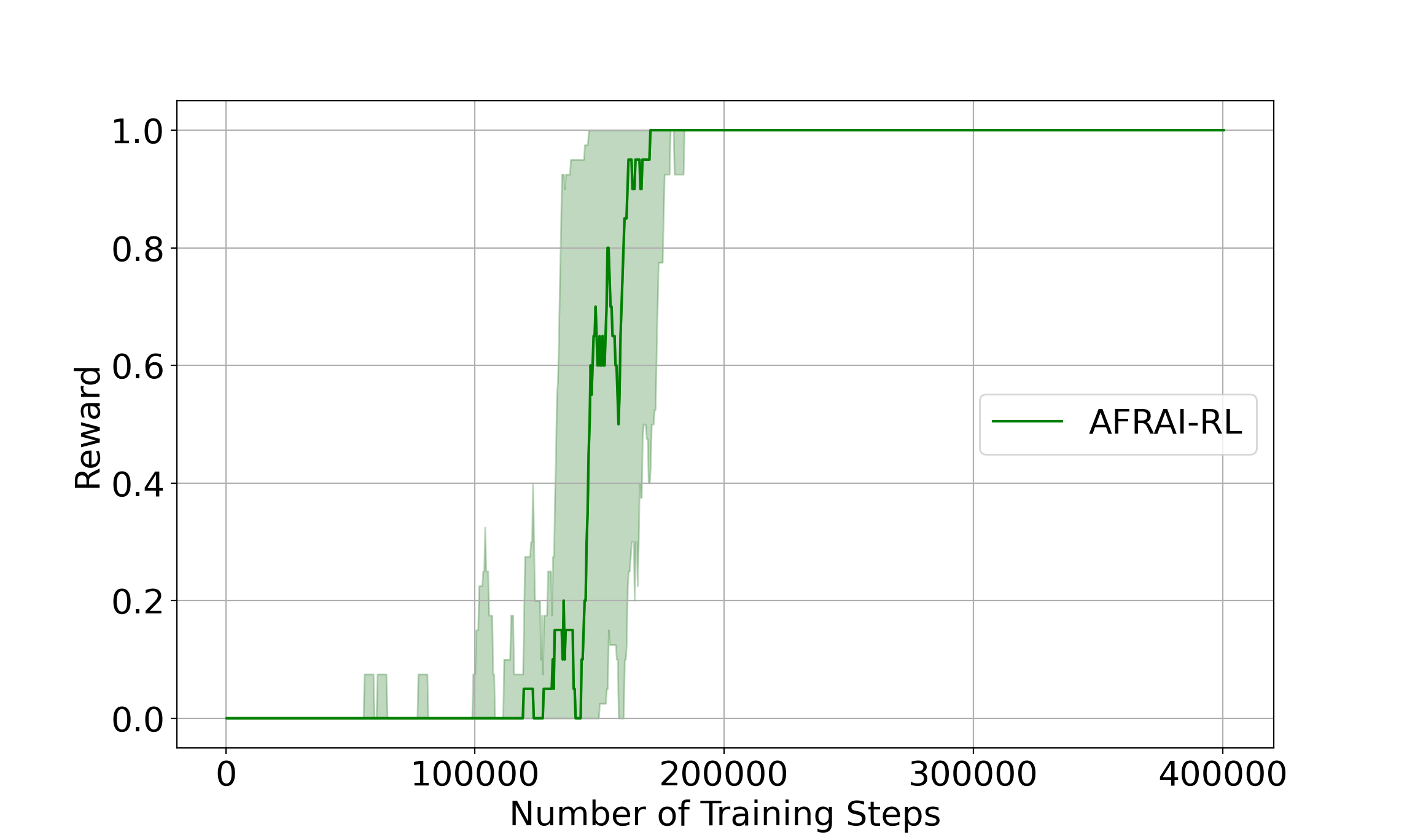}
\end{subfigure}
\begin{subfigure}[b]{\figurespacingCraft\textwidth}  
	\centering
	\includegraphics[width=\textwidth]{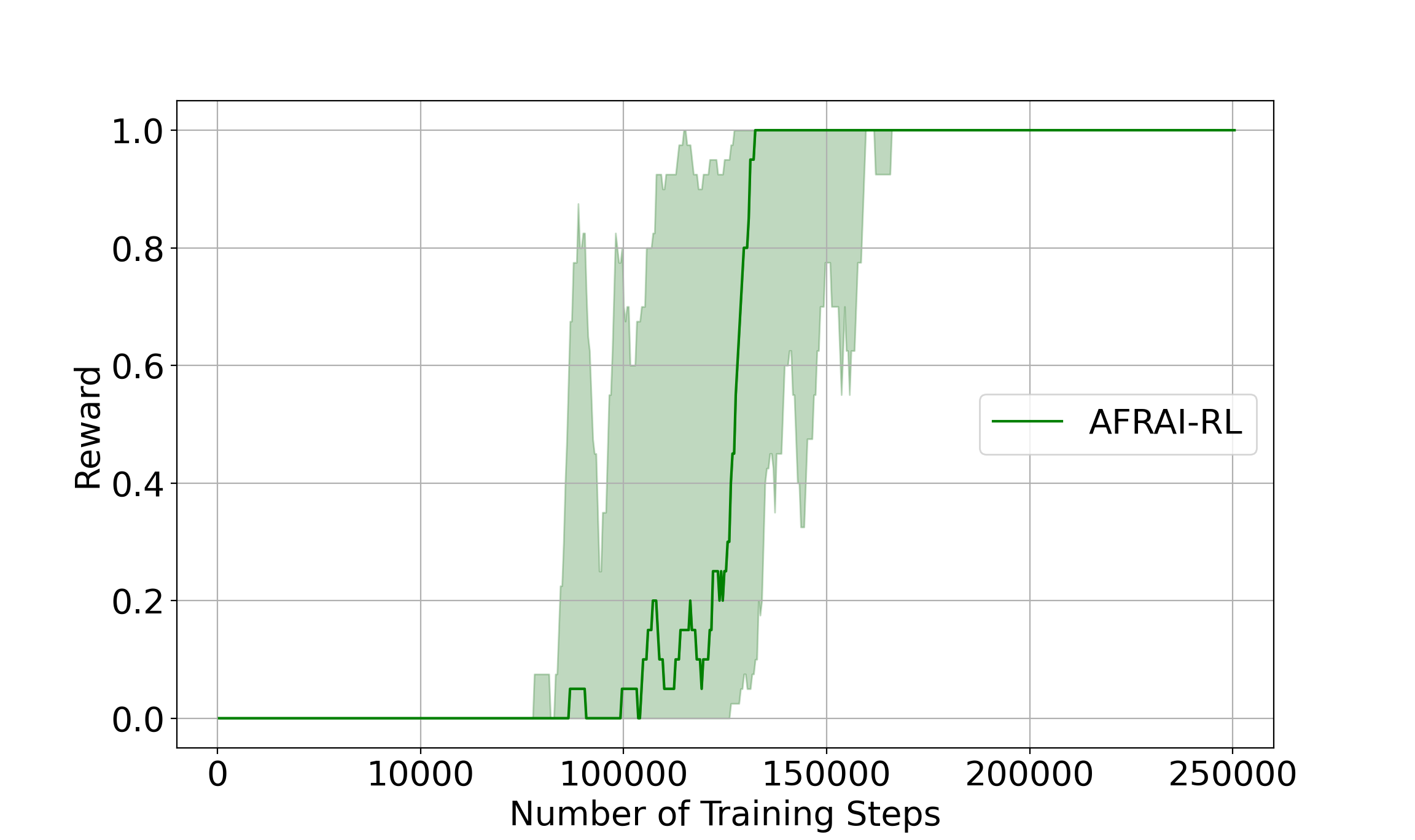}
\end{subfigure}
\begin{subfigure}[b]{\figurespacingCraft\textwidth}  
	\centering
	\includegraphics[width=\textwidth]{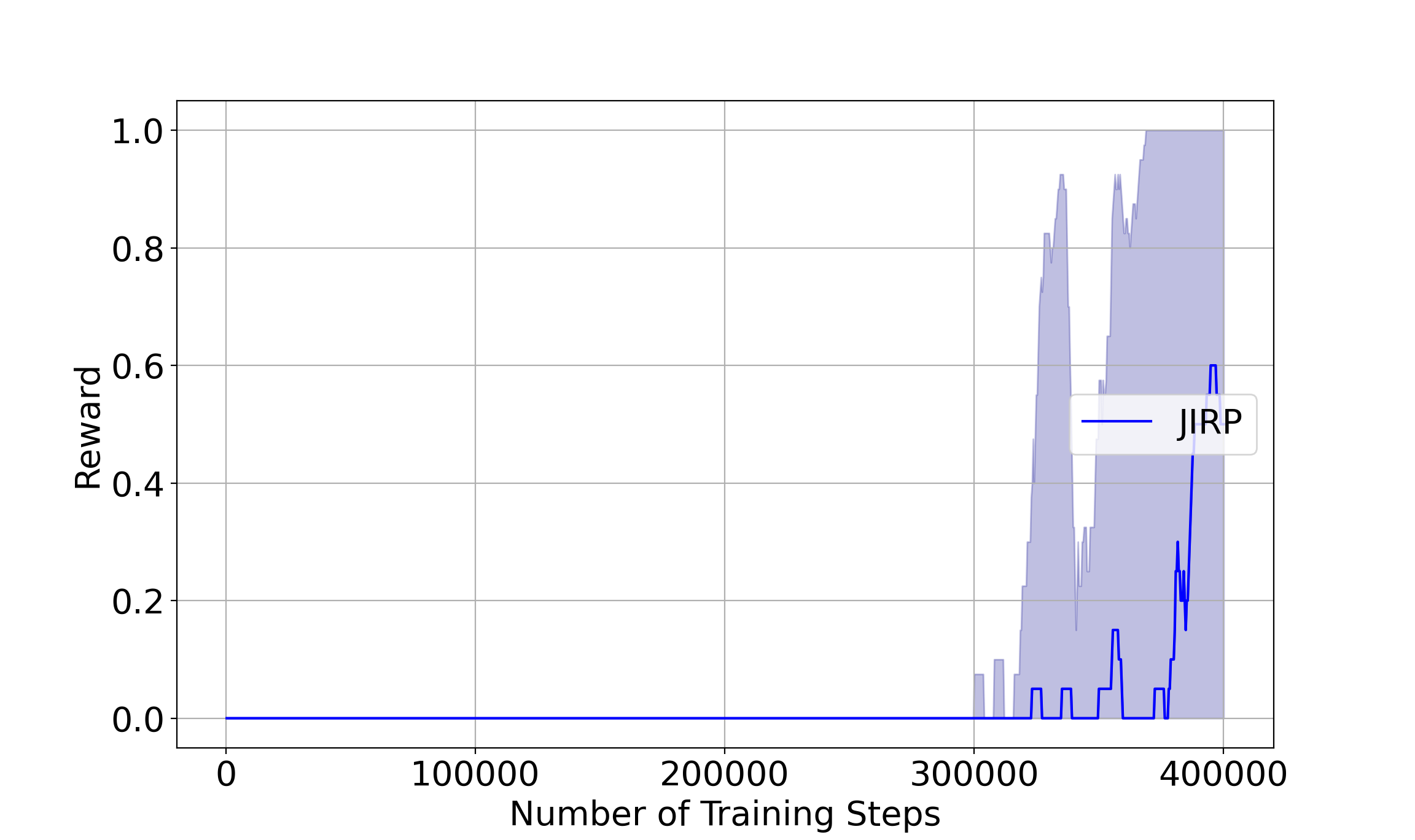}
	\end{subfigure}
\begin{subfigure}[b]{\figurespacingCraft\textwidth}  
	\centering
	\includegraphics[width=\textwidth]{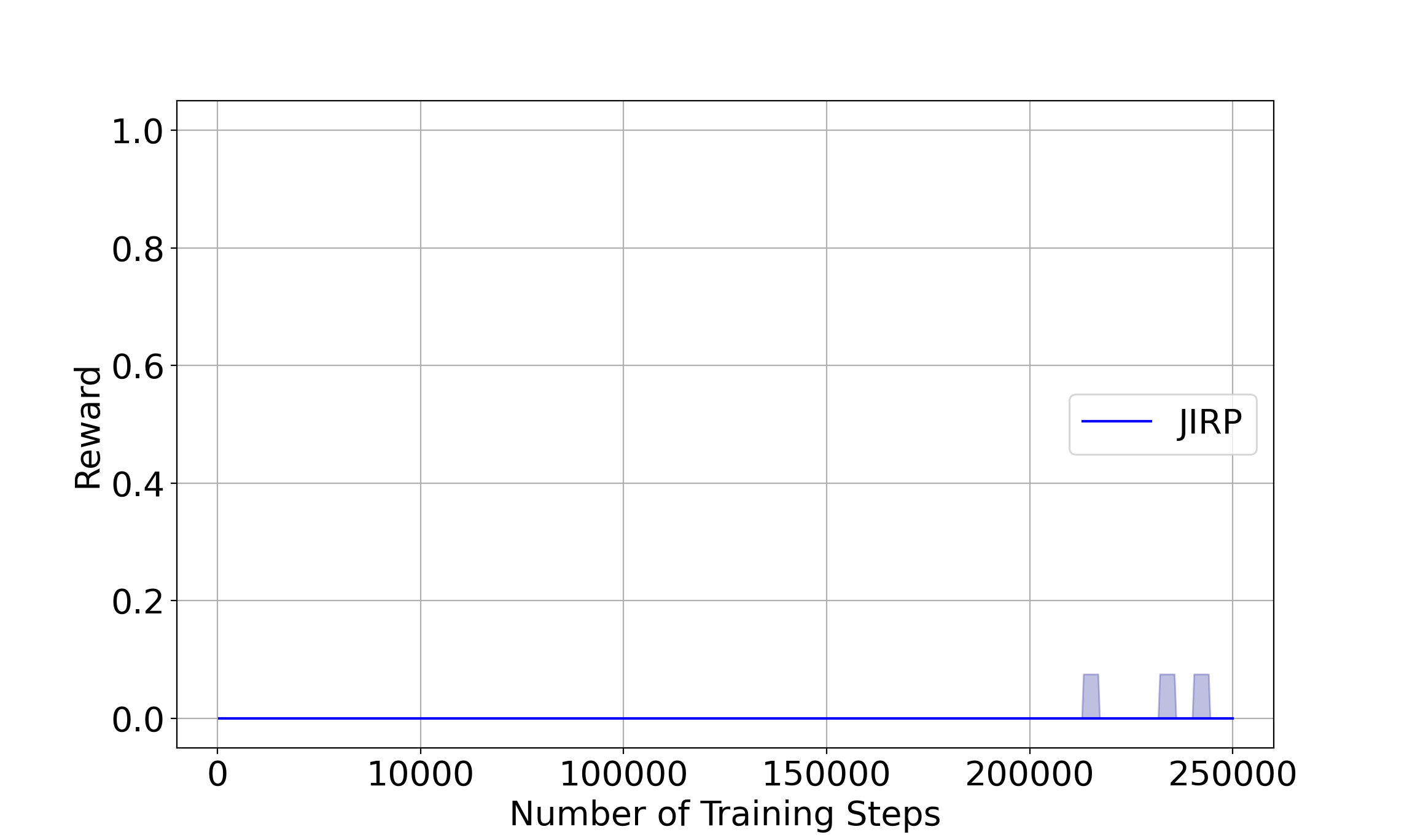}
	\end{subfigure}
\begin{subfigure}[b]{\figurespacingCraft\textwidth}  
	\centering
	\includegraphics[width=\textwidth]{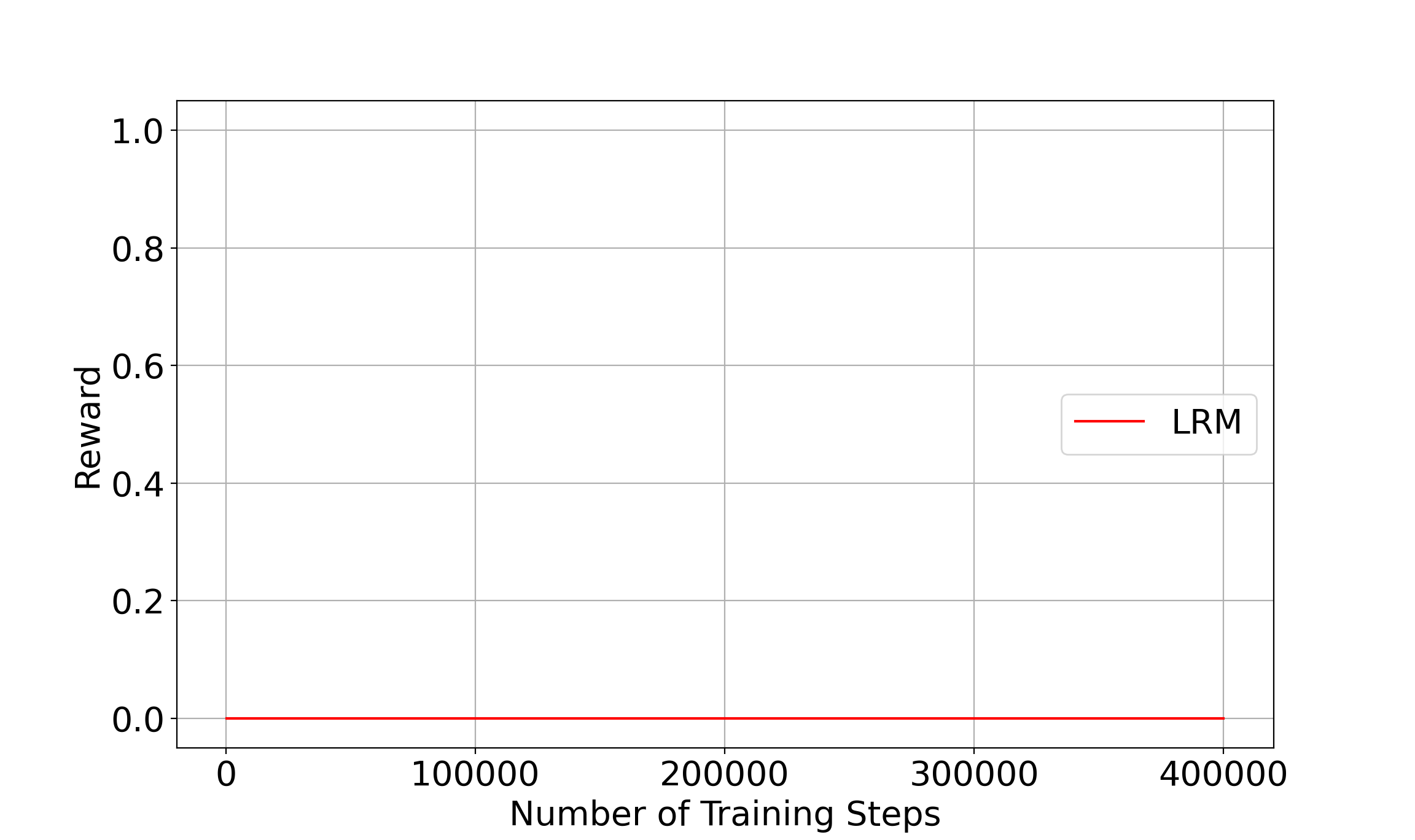}
	\end{subfigure}
\begin{subfigure}[b]{\figurespacingCraft\textwidth}  
	\centering
	\includegraphics[width=\textwidth]{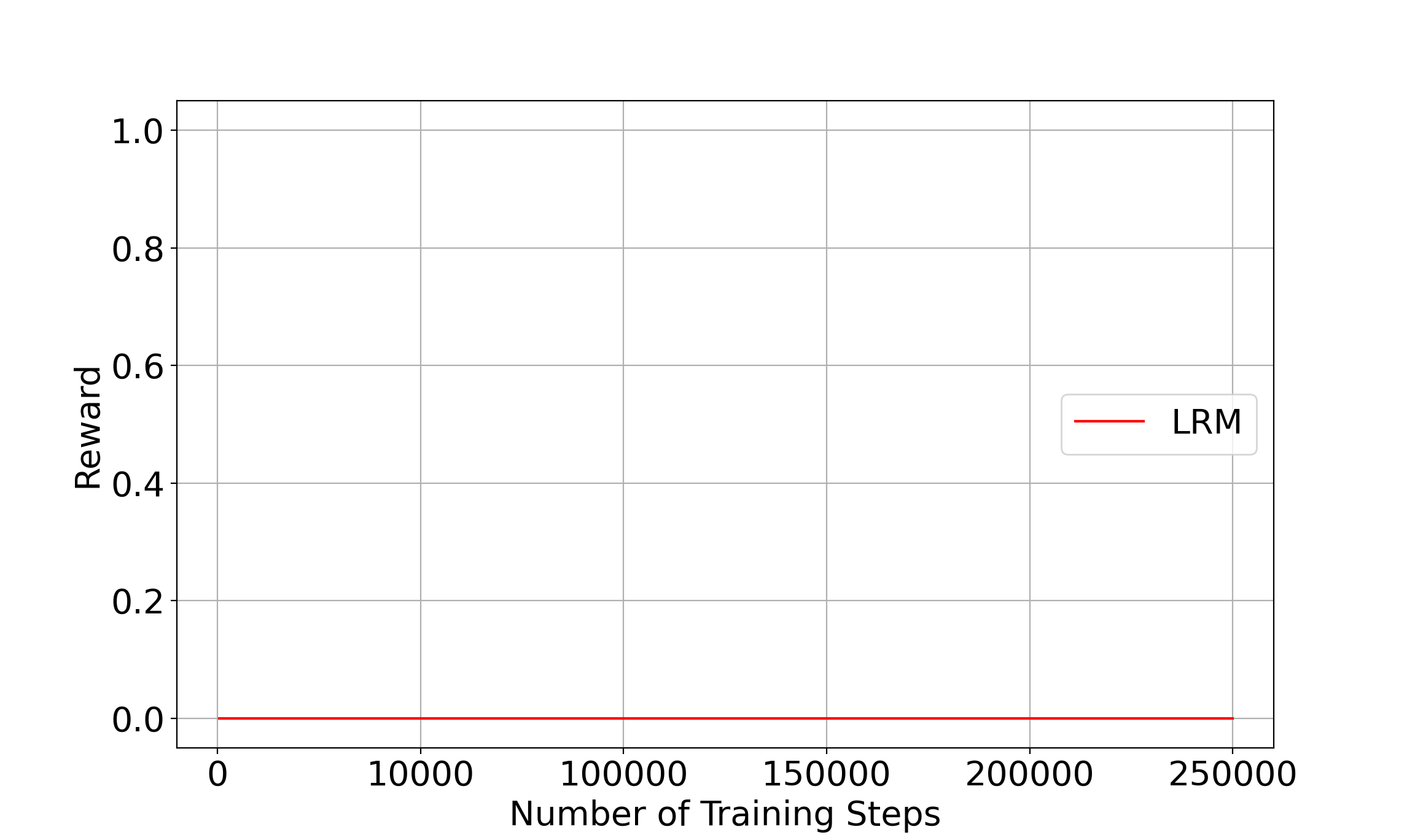}
\end{subfigure}
\begin{subfigure}[b]{\figurespacingCraft\textwidth}  
	\centering
	\includegraphics[width=0.9\textwidth]{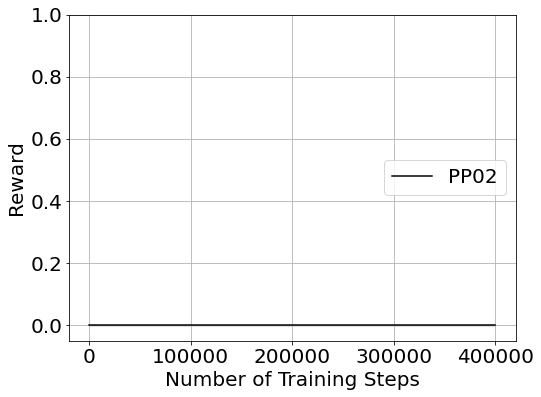}
	\caption{Task 1}
	\end{subfigure}
\begin{subfigure}[b]{\figurespacingCraft\textwidth}  
	\centering
	\includegraphics[width=0.9\textwidth]{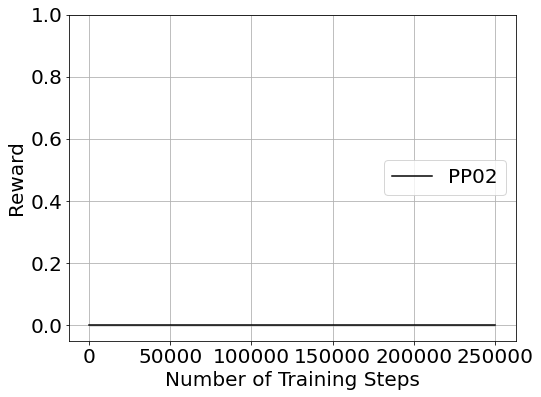}
	\caption{Task 2}
	\end{subfigure}
\caption{Attained rewards of 10 independent simulation runs of the \craft, averaged for every 10 training steps in \newAlgoName~(First row), \methodA-SAT (Second Row), LRM-QRM (Third Row) and PPO2 (Fourth Row): (a) Task 1; (b) Task 2.}  
\label{craft_tasks}
\end{figure*}
Figure \ref{craft_tasks} shows the attained rewards of 10 independent simulation runs for each task, averaged every 10 training steps. For task 1, it can be seen that, on average,
the proposed \newAlgoName\ approach converges to an optimal policy in about 190,000 training steps, while \methodA-SAT, LRM-QRM, and PPO2 do not converge to an optimal policy. For task 2, on average the proposed \newAlgoName\ approach converges to an optimal policy in about 170,000 million training steps, while \methodA-SAT, LRM-QRM, and PPO2 do not converge to an optimal policy.

\section{Conclusions}

We propose an active reinforcement learning approach that infers finite reward automata during the RL process through interaction with the environment. The algorithm can actively guide the RL towards finding answers to the queries needed for inferring the finite reward automata, thus making the finite reward automaton learning process more efficient. The case studies show that this algorithm is more efficient than recent baseline algorithms for the environments used. 

This work has potential in multiple future directions. First, we assume that the same state-action sequences lead to the same reward values in this paper (i.e., the method depends on the correct labeling of the trajectories). We will investigate the scenarios where incorrect labels can occur and the same state-action sequences may lead to different reward values. Second, current active RL work usually adopts a model-based framework and search for states that are less visited, thus improving the sample efficiency of RL. Our proposed approach adopts a model-free framework and actively recovers the reward structure, which turns out to be effective in the non-Markov RL scenarios. To make connections with the other work in active RL, we will investigate model-based RL with active finite reward automaton inference. Finally, as we consider RL for a single agent in this paper, we will extend the work to multi-agent settings for collaborative or non-collaborative games.


  \section{Acknowledgment}
This material is based upon work supported by the Defense Advanced Research Projects Agency (DARPA) under Contract No. HR001120C0032, ARL W911NF2020132, ARL ACC-APG-RTP W911NF, NSF 1646522, and Deutsche Forschungsgemeinschaft (DFG, German Research 
Foundation) - grant no. 434592664. Any
opinions, findings and conclusions or recommendations expressed in this material are those of
the author(s) and do not necessarily reflect the views of DARPA.

\bibliographystyle{IEEEtrans}
\bibliography{references}

\begin{thebibliography}{10}
\providecommand{\url}[1]{#1}
\csname url@samestyle\endcsname
\providecommand{\newblock}{\relax}
\providecommand{\bibinfo}[2]{#2}
\providecommand{\BIBentrySTDinterwordspacing}{\spaceskip=0pt\relax}
\providecommand{\BIBentryALTinterwordstretchfactor}{4}
\providecommand{\BIBentryALTinterwordspacing}{\spaceskip=\fontdimen2\font plus
\BIBentryALTinterwordstretchfactor\fontdimen3\font minus
  \fontdimen4\font\relax}
\providecommand{\BIBforeignlanguage}[2]{{%
\expandafter\ifx\csname l@#1\endcsname\relax
\typeout{** WARNING: IEEEtranS.bst: No hyphenation pattern has been}%
\typeout{** loaded for the language `#1'. Using the pattern for}%
\typeout{** the default language instead.}%
\else
\language=\csname l@#1\endcsname
\fi
#2}}
\providecommand{\BIBdecl}{\relax}
\BIBdecl

\bibitem{Aksaray2016}
D.~Aksaray, A.~Jones, Z.~Kong, M.~Schwager, and C.~Belta, ``Q-learning for
  robust satisfaction of signal temporal logic specifications,'' in \emph{IEEE
  CDC'16}, Dec 2016, pp. 6565--6570.

\bibitem{Alshiekh2018SafeRL}
M.~Alshiekh, R.~Bloem, R.~Ehlers, B.~K{\"o}nighofer, S.~Niekum, and U.~Topcu,
  ``Safe reinforcement learning via shielding,'' in \emph{AAAI'18}, 2018.

\bibitem{andreas2017modular}
J.~Andreas, D.~Klein, and S.~Levine, ``Modular multitask reinforcement learning
  with policy sketches,'' in \emph{Proceedings of the 34th International
  Conference on Machine Learning-Volume 70}.\hskip 1em plus 0.5em minus
  0.4em\relax JMLR. org, 2017, pp. 166--175.

\bibitem{angluin1987learning}
D.~Angluin, ``Learning regular sets from queries and counterexamples,''
  \emph{Information and computation}, vol.~75, no.~2, pp. 87--106, 1987.

\bibitem{Nasim2021}
\BIBentryALTinterwordspacing
N.~Baharisangari, J.-R. Gaglion, D.~Neider, U.~Topcu, and Z.~Xu,
  ``Uncertainty-aware signal temporal logic inference,'' 2021. [Online].
  Available: \url{https://arxiv.org/abs/2105.11545}
\BIBentrySTDinterwordspacing

\bibitem{libalfTool}
B.~Bollig, J.~Katoen, C.~Kern, M.~Leucker, D.~Neider, and D.~R. Piegdon,
  ``libalf: The automata learning framework,'' in \emph{Computer Aided
  Verification, 22nd International Conference, {CAV} 2010, Edinburgh, UK, July
  15-19, 2010. Proceedings}, 2010, pp. 360--364.

\bibitem{Bombara2016}
G.~Bombara, C.-I. Vasile, F.~Penedo, H.~Yasuoka, and C.~Belta, ``A decision
  tree approach to data classification using signal temporal logic,'' in
  \emph{Proc. HSCC'16}, 2016, pp. 1--10.

\bibitem{cai2021modular}
M.~Cai, M.~Hasanbeig, S.~Xiao, A.~Abate, and Z.~Kan, ``Modular deep
  reinforcement learning for continuous motion planning with temporal logic,''
  2021.

\bibitem{Fu2014ProbablyAC}
J.~Fu and U.~Topcu, ``Probably approximately correct {MDP} learning and control
  with temporal logic constraints,'' \emph{Robotics: Science and Systems}, vol.
  abs/1404.7073, 2014.

\bibitem{Furelos-Blanco_Law_Russo_Broda_Jonsson_2020}
\BIBentryALTinterwordspacing
D.~Furelos-Blanco, M.~Law, A.~Russo, K.~Broda, and A.~Jonsson, ``Induction of
  subgoal automata for reinforcement learning,'' \emph{Proceedings of the AAAI
  Conference on Artificial Intelligence}, vol.~34, no.~04, pp. 3890--3897, Apr.
  2020. [Online]. Available:
  \url{https://ojs.aaai.org/index.php/AAAI/article/view/5802}
\BIBentrySTDinterwordspacing

\bibitem{ATVA2021}
J.-R. Gaglion, D.~Neider, R.~Roy, U.~Topcu, and Z.~Xu, ``Learning linear
  temporal properties from noisy data: A maxsat-based approach,'' in \emph{ATVA
  2021, Gold Coast, Australia, October 18-22}, ser. Lecture Notes in Computer
  Science.\hskip 1em plus 0.5em minus 0.4em\relax Springer, 2021.

\bibitem{Gaon_Brafman_2020}
\BIBentryALTinterwordspacing
M.~Gaon and R.~Brafman, ``Reinforcement learning with non-markovian rewards,''
  \emph{Proceedings of the AAAI Conference on Artificial Intelligence},
  vol.~34, no.~04, pp. 3980--3987, Apr. 2020. [Online]. Available:
  \url{https://ojs.aaai.org/index.php/AAAI/article/view/5814}
\BIBentrySTDinterwordspacing

\bibitem{HOLZINGER202128}
\BIBentryALTinterwordspacing
A.~Holzinger, B.~Malle, A.~Saranti, and B.~Pfeifer, ``Towards multi-modal
  causability with graph neural networks enabling information fusion for
  explainable {AI},'' \emph{Information Fusion}, vol.~71, pp. 28--37, 2021.
  [Online]. Available:
  \url{https://www.sciencedirect.com/science/article/pii/S1566253521000142}
\BIBentrySTDinterwordspacing

\bibitem{Hopcroft2006}
J.~E. Hopcroft, R.~Motwani, and J.~D. Ullman, \emph{Introduction to Automata
  Theory, Languages, and Computation (3rd Edition)}.\hskip 1em plus 0.5em minus
  0.4em\relax Boston, MA, USA: Addison-Wesley Longman Publishing Co., Inc.,
  2006.

\bibitem{Hoxha2017}
\BIBentryALTinterwordspacing
B.~Hoxha, A.~Dokhanchi, and G.~Fainekos, ``Mining parametric temporal logic
  properties in model-based design for cyber-physical systems,''
  \emph{International Journal on Software Tools for Technology Transfer}, pp.
  79--93, Feb 2017. [Online]. Available:
  \url{http://dx.doi.org/10.1007/s10009-017-0447-4}
\BIBentrySTDinterwordspacing

\bibitem{IcarteNIPS2019}
R.~A.~T. Icarte, E.~Waldie, T.~Klassen, R.~Valenzano, M.~P. Castro, and S.~A.
  McIlraith, ``Learning reward machines for partially observable reinforcement
  learning,'' in \emph{NeurIPS}, 2019.

\bibitem{DBLP:conf/icml/IcarteKVM18}
R.~T. Icarte, T.~Q. Klassen, R.~A. Valenzano, and S.~A. McIlraith, ``Using
  reward machines for high-level task specification and decomposition in
  reinforcement learning,'' in \emph{{ICML} 2018, Stockholmsm{\"{a}}ssan,
  Stockholm, Sweden, July 10-15, 2018}, 2018, pp. 2112--2121.

\bibitem{Kong2017TAC}
Z.~Kong, A.~Jones, and C.~Belta, ``Temporal logics for learning and detection
  of anomalous behavior,'' \emph{IEEE TAC}, vol.~62, no.~3, pp. 1210--1222,
  Mar. 2017.

\bibitem{Li2017}
X.~Li, C.-I. Vasile, and C.~Belta, ``Reinforcement learning with temporal logic
  rewards,'' in \emph{Proc. IROS'17}, Sept 2017, pp. 3834--3839.

\bibitem{Neider2020AdviceGuidedRL}
D.~Neider, J.-R. Gaglione, I.~Gavran, U.~Topcu, B.~Wu, and Z.~Xu,
  ``Advice-guided reinforcement learning in a non-markovian environment,'' in
  \emph{AAAI'21}, 2021.

\bibitem{Neider}
D.~Neider and I.~Gavran, ``Learning linear temporal properties,'' in
  \emph{Formal Methods in Computer Aided Design (FMCAD)}, 2018, pp. 1--10.

\bibitem{schulman2017proximal}
J.~Schulman, F.~Wolski, P.~Dhariwal, A.~Radford, and O.~Klimov, ``Proximal
  policy optimization algorithms,'' 2017.

\bibitem{NIPS2018Shah}
\BIBentryALTinterwordspacing
A.~Shah, P.~Kamath, J.~A. Shah, and S.~Li, ``Bayesian inference of temporal
  task specifications from demonstrations,'' in \emph{NeurIPS}, S.~Bengio,
  H.~Wallach, H.~Larochelle, K.~Grauman, N.~Cesa-Bianchi, and R.~Garnett,
  Eds.\hskip 1em plus 0.5em minus 0.4em\relax Curran Associates, Inc., 2018,
  pp. 3808--3817. [Online]. Available:
  \url{http://papers.nips.cc/paper/7637-bayesian-inference-of-temporal-task-specifications-from-demonstrations.pdf}
\BIBentrySTDinterwordspacing

\bibitem{Icarte2018}
R.~Toro~Icarte, T.~Q. Klassen, R.~Valenzano, and S.~A. McIlraith, ``Teaching
  multiple tasks to an {RL} agent using {LTL},'' in \emph{AAMAS'18}, Richland,
  SC, 2018, pp. 452--461.

\bibitem{VazquezChanlatte2018LearningTS}
M.~Vazquez-Chanlatte, S.~Jha, A.~Tiwari, M.~K. Ho, and S.~A. Seshia, ``Learning
  task specifications from demonstrations,'' in \emph{NeurIPS}, 2018, pp.
  5372--5382.

\bibitem{Watkins1992}
\BIBentryALTinterwordspacing
C.~J. C.~H. Watkins and P.~Dayan, ``Q-learning,'' \emph{Machine Learning},
  vol.~8, no.~3, pp. 279--292, May 1992. [Online]. Available:
  \url{https://doi.org/10.1007/BF00992698}
\BIBentrySTDinterwordspacing

\bibitem{Min2017}
M.~Wen, I.~Papusha, and U.~Topcu, ``Learning from demonstrations with
  high-level side information,'' in \emph{Proc. IJCAI'17}, 2017, pp.
  3055--3061.

\bibitem{wu2015counterexample}
B.~Wu and H.~Lin, ``Counterexample-guided permissive supervisor synthesis for
  probabilistic systems through learning,'' in \emph{2015 American Control
  Conference (ACC)}.\hskip 1em plus 0.5em minus 0.4em\relax IEEE, 2015, pp.
  2894--2899.

\bibitem{wu2018permissive}
B.~Wu, X.~Zhang, and H.~Lin, ``Permissive supervisor synthesis for markov
  decision processes through learning,'' \emph{IEEE Transactions on Automatic
  Control}, vol.~64, no.~8, pp. 3332--3338, 2018.

\bibitem{zhang2017supervisor}
\BIBentryALTinterwordspacing
------, ``Supervisor synthesis of pomdp based on automata learning,''
  \emph{Automatica}, 2021, to appear. [Online]. Available:
  \url{https://arxiv.org/abs/1703.08262}
\BIBentrySTDinterwordspacing

\bibitem{zheletter2}
Z.~Xu, M.~Birtwistle, C.~Belta, and A.~Julius, ``A temporal logic inference
  approach for model discrimination,'' \emph{IEEE Life Sciences Letters},
  vol.~2, no.~3, pp. 19--22, Sept 2016.

\bibitem{zheADHS}
Z.~Xu, C.~Belta, and A.~Julius, ``Temporal logic inference with prior
  information: An application to robot arm movements,'' in \emph{Proc. Anal.
  and Des. of Hybrid Syst.}, vol.~48, no.~27, Atlanta, GA, USA, Oct., 2015, pp.
  141 -- 146.

\bibitem{Xu2019jirp}
Z.~Xu, I.~Gavran, Y.~Ahmad, R.~Majumdar, D.~Neider, U.~Topcu, and B.~Wu,
  ``Joint inference of reward machines and policies for reinforcement
  learning,'' in \emph{Proceedings of the International Conference on Automated
  Planning and Scheduling}, vol.~30, 2020, pp. 590--598.

\bibitem{zhe2016}
Z.~Xu and A.~Julius, ``Census signal temporal logic inference for multiagent
  group behavior analysis,'' \emph{IEEE Trans. Autom. Sci. Eng.}, vol.~15,
  no.~1, pp. 264--277, Jan. 2018.

\bibitem{zhe_info}
\BIBentryALTinterwordspacing
Z.~Xu, M.~Ornik, A.~Julius, and U.~Topcu, ``Information-guided temporal logic
  inference with prior knowledge,'' in \emph{Proc. IEEE Amer. Control Conf.},
  2019. [Online]. Available: \url{https://arxiv.org/abs/1811.08846}
\BIBentrySTDinterwordspacing

\bibitem{zhe_advisory}
Z.~Xu, S.~Saha, B.~Hu, S.~Mishra, and A.~Julius, ``Advisory temporal logic
  inference and controller design for semiautonomous robots,'' \emph{IEEE
  Trans. Autom. Sci. Eng.}, pp. 1--19, 2018.

\bibitem{zhe_ijcai2019}
\BIBentryALTinterwordspacing
Z.~Xu and U.~Topcu, ``Transfer of temporal logic formulas in reinforcement
  learning,'' in \emph{IJCAI-19}.\hskip 1em plus 0.5em minus 0.4em\relax
  International Joint Conferences on Artificial Intelligence Organization, 7
  2019, pp. 4010--4018. [Online]. Available:
  \url{https://doi.org/10.24963/ijcai.2019/557}
\BIBentrySTDinterwordspacing

\end{thebibliography}

\end{document}